\documentclass{article}

\usepackage[vmargin=1.18in,hmargin=1.1in]{geometry}
\usepackage{amsmath,amssymb,amsfonts,graphicx,amsthm,nicefrac,mathtools,bm}
\usepackage{hyperref}
\usepackage[numbers]{natbib}
\usepackage[english]{babel}
\usepackage[T1]{fontenc}
\usepackage{bbm,mdframed}
\usepackage{xcolor}
\usepackage{xspace}
\usepackage{rotating,multirow}
\usepackage{enumerate,graphics,enumitem}
\usepackage{caption}
\usepackage{booktabs}

\usepackage{gensymb}

\usepackage{thmtools}
\usepackage{thm-restate}
\usepackage{cleveref}

\declaretheorem[name=Lemma,numberwithin=section]{lemma}
\declaretheorem[name=Proposition,numberwithin=section]{prop}

\setlist[itemize]{noitemsep, topsep=0pt}
\setlist[enumerate]{itemsep=5pt, topsep=5pt, leftmargin=25pt}

\newtheorem{theorem}{Theorem}

\definecolor{verylightblue}{rgb}{0.7,0.8,1}
  {\begin{mdframed}[backgroundcolor=verylightblue]\begin{theorem}}%
  {\end{theorem}\end{mdframed}}

\definecolor{verylightgray}{gray}{0.95}
  {\begin{mdframed}[backgroundcolor=verylightgray]\begin{proof}}%
  {\end{proof}\end{mdframed}}

\definecolor{verylightred}{rgb}{1,0.8,0.8}
  {\begin{mdframed}[backgroundcolor=verylightred]\begin{lemma}}%
  {\end{lemma}\end{mdframed}}

\newtheorem{proposition}{Proposition}
  {\begin{mdframed}[backgroundcolor=verylightblue]\begin{proposition}}%
  {\end{proposition}\end{mdframed}}

\theoremstyle{definition}

\theoremstyle{remark}

\makeatletter

\definecolor{DarkRed}{rgb}{0.5,0.1,0.1}
\definecolor{DarkBlue}{rgb}{0.1,0.1,0.5}

\DeclareMathOperator*{\argmax}{arg\,max}
\DeclareMathOperator*{\argmin}{arg\,min}

\newcommand{\E}{\mathbb{E}}

\def\R{\mathbb{R}}

\renewcommand{\S}{\mathcal{S}}

\newcommand{\ignore}[1]{}

\usepackage{tikz}
\usetikzlibrary{arrows}
\usetikzlibrary{shapes}

\newcommand{\hDelta}{\widehat{\Delta}}

\newcommand{\ihat}{\hat{\imath}}
\newcommand{\thetaswap}{\theta^{\mathrm{swap}}}

\newcommand{\thedate}{\today}
\newcommand{\theauthor}{Tijana Zrnic$^1$ \qquad \qquad William Fithian$^2$\\
$^1$Stanford Data Science and Department of Statistics, Stanford University\\
$^2$Department of Statistics, University of California, Berkeley
}

\newcommand{\thetitle}{A Flexible Defense Against the Winner's Curse}

\date{\thedate}
\author{\theauthor}
\title{\thetitle}

\def\X{X}

\def\I{\mathcal{I}}

\def\tvec{\mathbf{t}}
\def\rvec{\mathbf{r}}

\makeatletter
\newcommand{\dotfrac}[2]{
\mathchoice
{\ooalign{$\genfrac{}{}{0pt}{0}{#1}{#2}$\cr\leavevmode\cleaders\hb@xt@ .22em{\hss $\displaystyle\cdot$\hss}\hfill\kern\z@\cr}}
{\ooalign{$\genfrac{}{}{0pt}{1}{#1}{#2}$\cr\leavevmode\cleaders\hb@xt@ .22em{\hss $\textstyle\cdot$\hss}\hfill\kern\z@\cr}}
{\ooalign{$\genfrac{}{}{0pt}{2}{#1}{#2}$\cr\leavevmode\cleaders\hb@xt@ .22em{\hss $\scriptstyle\cdot$\hss}\hfill\kern\z@\cr}}
{\ooalign{$\genfrac{}{}{0pt}{3}{#1}{#2}$\cr\leavevmode\cleaders\hb@xt@ .22em{\hss $\scriptscriptstyle\cdot$\hss}\hfill\kern\z@\cr}}
}
\makeatother

\usepackage{algorithm,algorithmic}

\makeatletter
\long\def\@makecaption#1#2{
        \vskip 0.8ex
        \setbox\@tempboxa\hbox{\small {\bf #1:} #2}
        \parindent 1.5em  
        \dimen0=\hsize
        \advance\dimen0 by -3em
        \ifdim \wd\@tempboxa >\dimen0
                \hbox to \hsize{
                        \parindent 0em
                        \hfil 
                        \parbox{\dimen0}{\def\baselinestretch{0.96}\small
                                {\bf #1.} #2
                                } 
                        \hfil}
        \else \hbox to \hsize{\hfil \box\@tempboxa \hfil}
        \fi
        }
\makeatother

\usepackage[algo2e,lined,ruled]{algorithm2e}

\begin{document}

\maketitle

\begin{abstract}
Across science and policy, decision-makers often need to draw conclusions about the best candidate among competing alternatives. For instance, researchers may seek to infer the effectiveness of the most successful treatment or determine which demographic group benefits most from a specific treatment. Similarly, in machine learning, practitioners are often interested in the population performance of the model that performs best empirically. However, cherry-picking the best candidate leads to the winner's curse: the observed performance for the winner is biased upwards, rendering conclusions based on standard measures of uncertainty invalid. We introduce the \emph{zoom correction}, a novel approach for valid inference on the winner. Our method is flexible: it can be employed in both parametric and nonparametric settings, can handle arbitrary dependencies between candidates, and automatically adapts to the level of selection bias. The method easily extends to important related problems, such as inference on the top $k$ winners, inference on the value and identity of the population winner, and inference on ``near-winners.''
\end{abstract}

\section{Introduction}

Data science and machine learning frequently grapple with selection bias. Whether it is reporting the largest treatment effect, choosing the model hyperparameters that perform best empirically, or identifying the most successful intervention, choosing ``the winner'' tends to result in overly optimistic conclusions. For example, the largest observed treatment effect is likely to overestimate the true effect; excessive tuning of model hyperparameters may result in poor population performance. To ensure trustworthy conclusions, it is essential to quantify and correct for the bias introduced by the selection.

Mathematically, the problem can be framed as follows: given noisy estimates $X_1,\dots,X_m$, for $m$ different candidates, infer the mean of estimate $\ihat = \argmax_{i\in[m]} X_i$. For example, $X_i$ may represent the observed effect of treatment $i$ or the empirical performance of model $i$. Inference on the winner's mean must account for the {\em winner's curse}: that $X_{\ihat}$ tends to be systematically biased upwards, especially when there are a large number of close competitors for the title of winner.   

Although the winner's curse is well-known, the statistical problem of correcting for it remains an active area of research, with significant progress made in recent years  \cite{andrews2019inference, andrews2022inference, benjamini2019confidence,fuentes2018confidence, zrnic2022locally}. Moreover, there is still room for improvement. Designing methods that avoid strong distributional assumptions, such as independence or Gaussianity of $X_i$, and that adapt to the selection bias in the actual data rather than assuming the worst case, has proven challenging.

We seek to address these challenges by introducing the \emph{zoom correction}, a method for valid inference on the winner that is flexible in its applicability and adaptive to the data at hand. Our method can exploit knowledge of the error distribution if it is known, but can also operate under nonparametric assumptions with unknown, arbitrary dependence between the $X_i$. It is adaptive in the sense that it adjusts for the selection bias in the problem at hand, rather than a hypothetical worst case, by making  a more modest adjustment when there are fewer competitive candidates. In particular, as the gap between the winner and all other candidates increases, the method approaches standard, uncorrected inference.

To briefly sketch the method, it first forms a simultaneous confidence region for all $m$ candidates and then simply projects this set along the winning coordinate $\ihat$ to obtain a confidence interval for the winner. The key innovation lies in a careful design of the simultaneous confidence region, such that the subsequent projection is both computationally tractable and tight. We form the confidence region by inverting what we call the \emph{zoom test}, a hypothesis test that ``zooms in''---i.e., focuses its statistical power---around the best candidates.

Given its robust foundations, our method can be easily generalized and extended to several problems related to inference on the winner. These include inference on the top $k$ winners, inference on the value and identity of the population winner, and inference on ``near-winners.''

\subsection{Problem setup}

We introduce the problem formally. Let $\X = (X_1,\dots,X_m)$ be an $m$-dimensional random vector of outcomes, and let $\theta = (\theta_1,\dots,\theta_m)= \E[X]$ denote its mean. Our goal is to infer $\theta_{\ihat}$, where $\ihat$ is the index of the \emph{winner}:
\begin{equation}
\label{eq:main_problem}
\ihat = \argmax_{i\in[m]} \X_i.
\end{equation}
That is, we want a confidence interval $\widehat C_{\ihat}^\alpha$ such that $P(\theta_{\ihat} \in \widehat C_{\ihat}^\alpha)\geq 1-\alpha$ for a target level $1-\alpha$.

Our main working assumption is a known tail bound on the errors $\xi = X-\theta$;
specifically, we assume that for any vector $v\in\R_+^m$, we have a continuous, coordinate-wise decreasing function $\mathbf{S}(v)$ such that
\[P\left(\exists i \in[m]:|\xi_i|> v_i\right) \leq \mathbf{S}(v).\]
If the error distribution is known, $\mathbf{S}(v)$ can be calculated exactly; otherwise it can be obtained by applying a union bound to marginal error bounds: $\mathbf{S}(v) = \sum_{i=1}^m S_i(v_i)$, where $S_i(v_i) \geq P(|\xi_i|> v_i)$. In the second construction, we need make no assumptions about the dependence between the errors $\xi_j$.

To keep the exposition clear, the problem \eqref{eq:main_problem} will be the focus of our main technical sections. However, our tools are also applicable to related problems. In Section \ref{sec:extensions} we provide several generalizations and extensions of our results, including to inference on the top $k$ winners, inference on the value $\theta^* = \max_{i\in[m]} \theta_i$ and identity $i^* = \argmax_{i\in[m]} \theta_i$ of the \emph{population} winner, and inference on ``near-winners.''

\subsection{Simultaneous inference}
The key methodological framework underpinning our method is \emph{simultaneous inference}~\cite{miller1981simultaneous, dickhaus2014simultaneous}. This is a general criterion that \emph{all} possible inferential targets be covered simultaneously; in our case, this corresponds to covering all coordinates of $\theta$. Formally, a set $\widehat C^\alpha$ is a simultaneous confidence region for $\theta\in\R^m$ at level $1-\alpha$ if
\[P(\theta \in \widehat C^\alpha)\geq 1-\alpha.\]
One standard way of obtaining $\widehat C^\alpha$ is to perform a Bonferroni correction, a.k.a. a union bound: $\widehat C^\alpha = \widehat C_{1}^{\alpha/m} \times \widehat C_{2}^{\alpha/m} \times \dots \times \widehat C_m^{\alpha/m}$, where $\widehat C_i^q$ denotes any confidence interval for $\theta_i$ valid marginally at level~$1-q$.

Given any such confidence region, we can perform inference on the winner \eqref{eq:main_problem} simply by returning the projection of $\widehat C^\alpha$ along coordinate $\ihat$: $\widehat C_{\ihat}^\alpha = \{\theta_{\ihat}: \theta \in \widehat C^\alpha\}$. By implication, we know 
\[P(\theta_{\ihat} \in \widehat C^\alpha_{\ihat}) \geq P(\theta \in \widehat C^\alpha) \geq 1-\alpha.\]
This is a major appeal of simultaneous inference: no matter what property $g(\theta)$ we want to infer, we know that $\widehat C^\alpha_g = \{g(\theta) : \theta\in \widehat C^\alpha\}$ covers the target of interest---even if $g$ is data-dependent. This flexibility is what allows our approach to be applicable to a number of related problems, including inference on the top $k$ winners, inference on the value and index of the population winner, and so on.

Another appealing aspect of simultaneous inference is that it can be performed in very general, nonparametric settings. For example, as long as we know how to construct valid marginal intervals, the Bonferroni correction is valid under no further assumptions.

\subsection{Related work}

While the other problems in Section \ref{sec:extensions} are not as well studied, there have been a number of prior approaches to inference on the winner and the top $k$ winners, which we review below.

\paragraph{Solutions based on ``focusing'' simultaneous inference.}
There are several works that study inference on winners by ``focusing'' simultaneous inference \cite{venter1988confidence, fuentes2018confidence, benjamini2019confidence, zrnic2022locally}. These works can be interpreted as providing simultaneously valid intervals for all coordinates of $\theta$, such that the intervals are tight for the selected coordinates and infinite for the non-selected coordinates. Our work similarly focuses power around the winner, however it comes with the option of returning finite intervals for non-winners as well. The intervals are reasonably tight and informative for ``near-winners.''

Most related to our work is \emph{locally simultaneous inference} (LSI) \cite{zrnic2022locally}. Like our approach, LSI provides simultaneous inference for $\theta$ while focusing power around the selected candidate $\ihat$. However, our approach is carefully tailored to the problem of inference on the winner and is thus more powerful than LSI, which is a more general paradigm. Moreover, LSI relies on choosing an error budget splitting parameter, and because of this error split it does not recover uncorrected inference when the winner stands out. Our approach comes with no tuning parameters; it automatically adapts to the difficulty of the problem and recovers uncorrected inference in the extreme case.

The other approaches \cite{venter1988confidence,fuentes2018confidence,benjamini2019confidence}, designed specifically for inference on the top $k$ winners, require independent observations and their confidence interval widths are not adaptive to the data at hand. For example, they do not depend on how obvious the winner is, and they do not approach uncorrected inference when the non-winners are clearly suboptimal.

\paragraph{Conditional solutions.}
Another family of solutions is based on \emph{conditional inference} \cite{fithian2014optimal}. These approaches rely on characterizing the distribution of $X$ conditional on the event $\{\ihat = i\}$ for any given $i$. Since this characterization is non-trivial in general, conditional approaches focus on parametric exponential families. The event $\{\ihat = i\}$ can be formulated as a polyhedral constraint on $X$. Therefore, in the Gaussian case one can apply conditional inference based on the polyhedral lemma \cite{lee2016exact}. However, the intervals obtained through this approach are often very large, especially when there are several competitive candidates; in fact, \citet{kivaranovic2021length} showed that these intervals have infinite expected length.
\citet{andrews2019inference} introduced a refinement of the polyhedral approach, called \emph{hybrid inference}. The hybrid method begins by constructing simultaneous intervals for all candidates. Then, it implements a correction conditional on both the selected target and the event that the intervals constructed in the first step cover the target. Hybrid inference leads to significant power gains over the polyhedral approach, though it still relies on Gaussianity with a known (or estimable) covariance. Like LSI, the hybrid approach also comes with an error budget splitting parameter, which prevents it from recovering uncorrected inference when the winner stands out.

\paragraph{Randomized approaches.} Finally, one can apply randomized approaches to perform inference on the winner, including data splitting and noise addition \cite{tian2018selective,zrnic2020post, rasines2023splitting, leiner2023data, neufeld2024data}. These approaches allow trading off the selection quality for an increase in statistical power. Instead of picking the exact winner, these methods select a ``noisy winner.'' Of course, randomization diminishes the quality of the selected candidate. We focus on exact, non-randomized selection of the winner.

\section{Zoom test}
\label{sec:test}

At a high level, our approach consists of the following steps:
\begin{enumerate}
\item Define a point hypothesis test for the mean of $X$, $H_0(\theta):\E[X]=\theta$, for each $\theta \in \R^m$.
\item Invert the test to obtain the simultaneous confidence region $\widehat C^\alpha(X) = \{\theta \in \R^m:\; H_0(\theta) \text{ not rejected}\}$.
\item Project $\widehat C^\alpha$ along the winning coordinate: $\widehat C_{\ihat}^\alpha = \{\theta_{\ihat}:\theta \in \widehat C^\alpha\}$.
\end{enumerate}
In this section we define the hypothesis test, which we call the \emph{zoom test}. In the following section we will show how to invert it and project the set along coordinate $\ihat$.

The zoom test at a fixed $\theta\in\R^m$ is a test for the null $H_0(\theta): \E[X] = \theta$. 
We denote by $A_\alpha(\theta)$ the acceptance region of the test at level $1-\alpha$.
When $\theta$ is clear from the context, we will keep the dependence on $\theta$ implicit and simply write $H_0$ and $A_\alpha$. We will do the same for the other quantities defined in this section: in the first mention we will expose their dependence on $\theta$, after which we will silently drop $\theta$ from the notation.

We construct $A_\alpha$ so as to guarantee that $P(\X \in A_\alpha) \geq 1-\alpha$.
To define $A_\alpha$, let $\Delta_j(\theta) = \theta^* - \theta_j$ be the {\em suboptimality} of $\theta_j$, where $\theta^* = \max_{i\in[m]} \theta_i$. Note that $\min_j \Delta_j = 0$, since the population winner has suboptimality zero.
Intuitively, the idea behind $A_\alpha$ is to focus the statistical power around the most promising candidates. This is done by forcing $A_\alpha$ to be large around coordinates $j$ with large $\Delta_j$, allowing us to allocate the error budget only to the coordinates $j$ with small $\Delta_j$. In other words, the method seeks to ignore the suboptimal coordinates of $\theta$, effectively reducing the multiplicity of the problem.

For a carefully chosen value $r_\alpha(\theta)$, we define the region
\begin{align*}
A_\alpha &= \left[\theta \pm \left(r_\alpha\vee\frac \Delta 2\right)\right]\\
&= \left[\theta_1 - \left(r_\alpha \vee \frac{\Delta_1}{2}\right), \theta_1 + \left(r_\alpha \vee \frac{\Delta_1}{2}\right)\right] \times \cdots \times \left[\theta_m - \left(r_\alpha \vee \frac{\Delta_m}{2}\right), \theta_m + \left(r_\alpha \vee \frac{\Delta_m}{2}\right)\right],
\end{align*}
where the radius in coordinate $j$ is the \emph{maximum} of $r_\alpha$ and the suboptimality: $(r_\alpha\vee\frac{\Delta_j}{2}) = \max\{r_\alpha, \frac{\Delta_j}{2}\}$. We choose $r_\alpha$ to ensure that $A_\alpha$ is indeed a $1-\alpha$ level acceptance region:
\begin{equation}
\label{eq:active-radius}
r_\alpha = \min \left\{r:\; \mathbf{S}\left(r\vee\frac \Delta 2\right) \leq \alpha \right\}.
\end{equation}
We call $r_\alpha$ the {\em active radius}.
Thus, $A_\alpha$ is a rectangle centered at $\theta$ whose width in its $j$th coordinate is forced to be no less than the suboptimality of candidate $j$. Define the {\em active indices} $\I_\alpha(\theta)$ as the set of indices whose width is $r_\alpha$:
\[
\I_\alpha = \{j:\; \Delta_j \leq 2r_\alpha\}.
\]
We let $\I_\alpha^c = [m]\setminus \I_\alpha$ denote the inactive indices.  It should be clear from the definition that $r_\alpha$, which depends on $\theta$ only through $\Delta$, is non-increasing in $\alpha$ and each $\Delta_j$. At one extreme, where $\Delta \equiv 0$ (that is, if $\theta_1=\cdots=\theta_m$), $r_\alpha$ coincides with the radius of the fully simultaneous confidence intervals for $\theta_1,\cdots,\theta_m$. At the other extreme, where there is a single maximum $\theta_k> \max_{j\neq k}\theta_j$ and $\min_{j\neq k} \Delta_j \to \infty$, $r_\alpha$ converges to the radius of a marginal confidence interval for $\theta_k$.
By construction, we have the following.

\begin{prop}[Zoom test]\label{prop:test}
The test that accepts when $X\in A_\alpha(\theta)$ and rejects otherwise is a valid test for the null $H_0:\E[X]=\theta$ at level $1-\alpha$.
\end{prop}

To provide further intuition, we depict the acceptance region $A_\alpha$ in Figure \ref{fig:A_alpha}. One can think of $A_\alpha$ as starting from a large radius $r$ and gradually shrinking it until the coverage equals $1-\alpha$. When the radius becomes equal to $\frac{\Delta_j}{2}$ for some $j$, that coordinate becomes inactive and its interval no longer shrinks. The remaining active intervals keep shrinking. The red intervals in the figure correspond to inactive coordinates; the blue intervals correspond to active coordinates. Notice that a coordinate $j$ is inactive if and only if its confidence interval does not overlap with the interval around the population winner, i.e. $\theta_j + \frac{\Delta_j}{2} < \theta^* - r_\alpha$. Intuitively, this suggests that it is unlikely for $X_j$ to win. Moreover, the inactive intervals increase in width with the suboptimality of the coordinate. This is why we call the test the zoom test: it ``zooms in'' around the largest coordinates. In the extreme case where $k$ coordinates are active and the remaining $m-k$ are inactive and highly suboptimal, $A_\alpha$ devotes a negligible portion of the error budget to those coordinates and essentially splits the entire error budget equally among the $k$ active candidates.

When the noise distribution is known exactly, meaning $\mathbf{S}(v) = P(\exists i\in[m] : |\xi_i|>v_i)$, a helpful characterization of $r_\alpha$ can be obtained by defining the random variable
\[
M(\xi,\Delta) \;=\; \max_{j\in[m]} |\xi_j|\,\cdot\,1\left\{|\xi_j|> \frac{\Delta_j}{2}\right\}.
\]

\begin{restatable}{prop}{teststat}
\label{prop:test-stat}
The active radius $r_{\alpha}$ is equal to the $1-\alpha$ quantile of $M(\xi,\Delta)$.
\end{restatable}

\noindent Therefore, when the noise distribution is known, we can evaluate $r_{\alpha}$ by simply simulating $M(\xi,\Delta)$ via repeated sampling from the noise distribution. When it is known, using the exact joint distribution of $\xi$ can improve significantly on the Bonferroni approach, especially if the errors are highly correlated.

\begin{figure}
\centering
\includegraphics[width=0.32\textwidth]{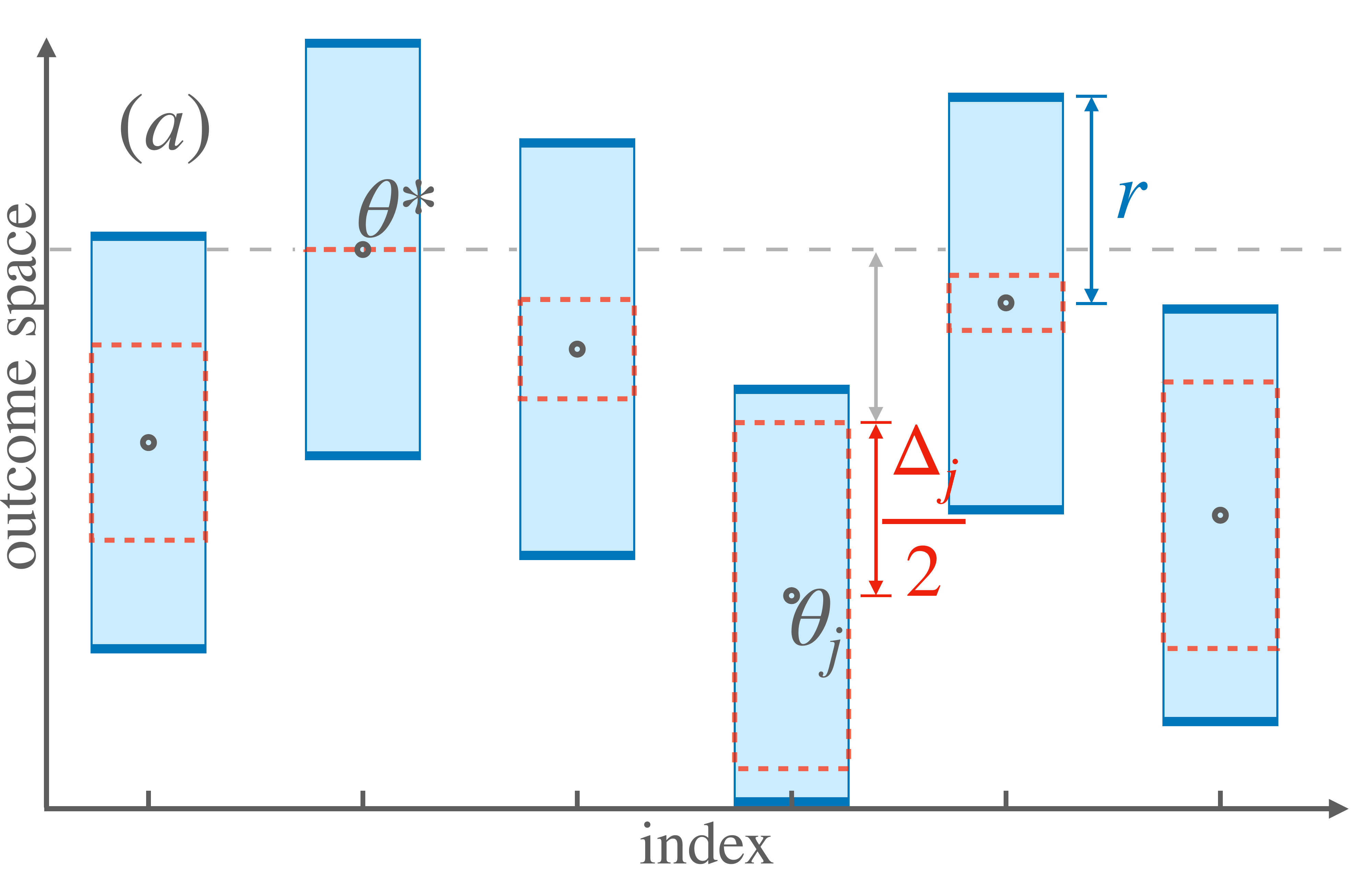}
\includegraphics[width=0.32\textwidth]{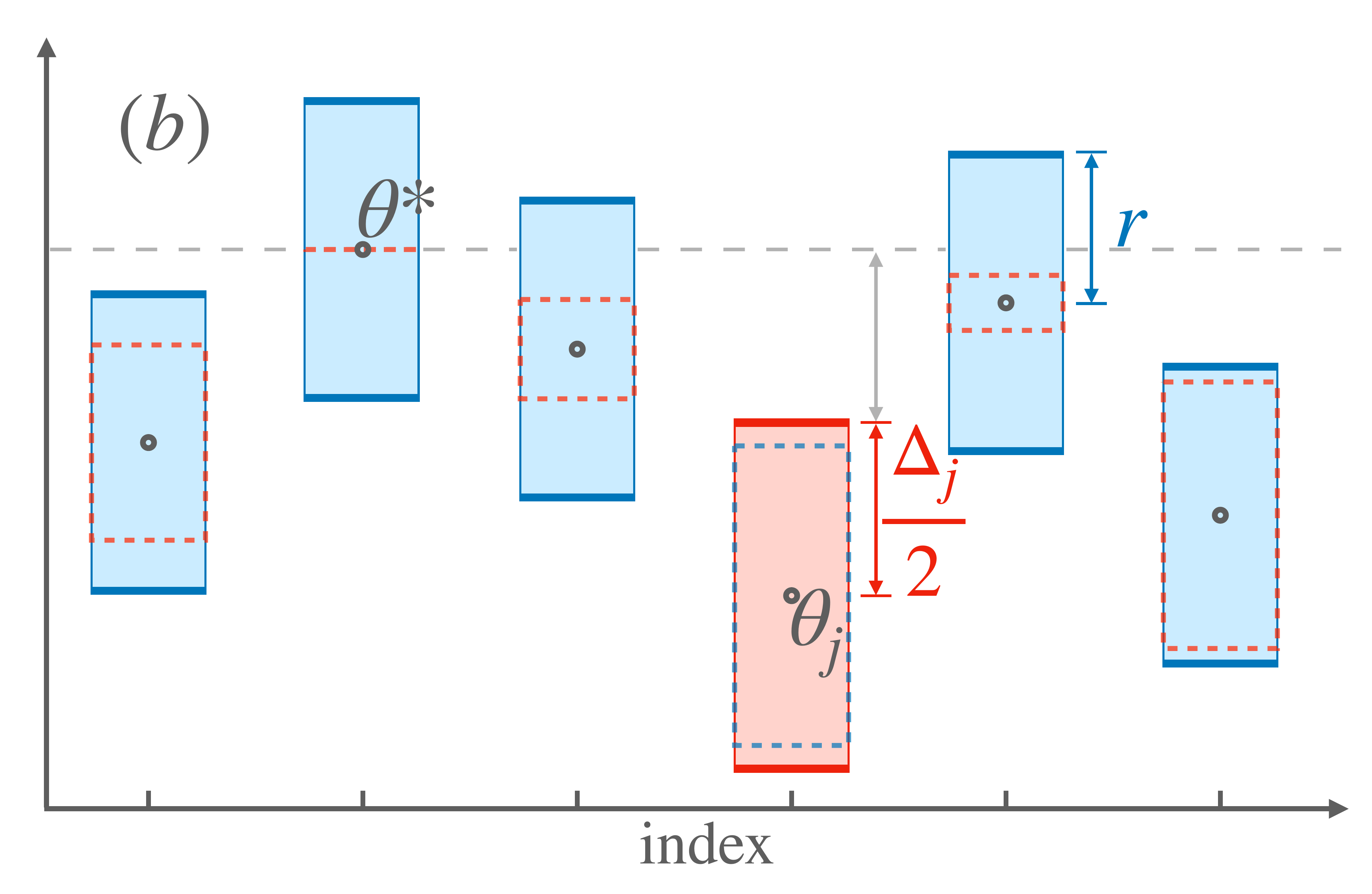}
\includegraphics[width=0.32\textwidth]{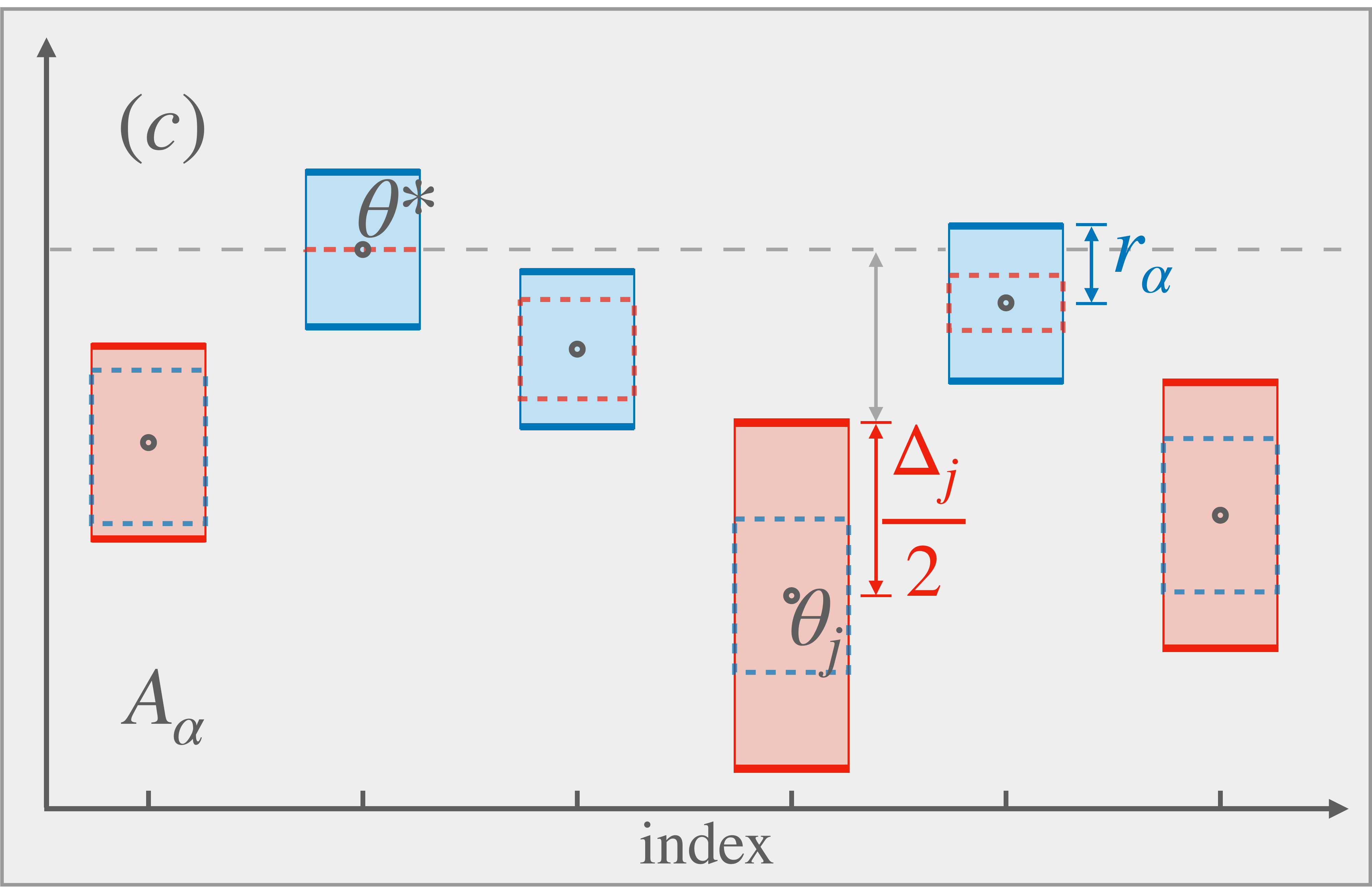}
\caption{\textbf{Definition of the acceptance region $A_\alpha$.} One can think of the acceptance region as starting from a large radius $r$ and shrinking it until the coverage equals $1-\alpha$. (a) We form boxes of radius $\frac{\Delta_j}{2}$ around each candidate $\theta_j$. The radius of $A_\alpha$ around $\theta_j$ can never be smaller than that box. At the beginning, $r$ is larger than all boxes. (b)
Once $r$ becomes equal to $\frac{\Delta_j}{2}$ for some $j$, that coordinate becomes inactive and its interval no longer shrinks. The remaining active intervals keep shrinking. (c) Once the coverage of the region becomes equal to $1-\alpha$, we stop shrinking the radius. The final radius is equal to $r_\alpha$, and $A_\alpha$ is the final region.}
\label{fig:A_alpha}
\end{figure}

\section{Confidence interval for the winner}
\label{sec:conf_int_winner}

We next construct a confidence interval for the winner by inverting the zoom test and projecting the resulting set along coordinate $\ihat$. We refer to this approach as the \emph{zoom correction}. 
Formally, define the confidence region obtained by inverting $A_\alpha$ to be
\[
\widehat C^\alpha = \{\theta : \X \in A_\alpha(\theta)\}.
\]
By the validity of the zoom test, we know that the true mean $\theta$ is covered by this set: $P(\theta \in \widehat C^\alpha) \geq 1-\alpha$.
If \emph{all} of $\theta$ is covered by $\widehat C^\alpha$, this means that, in particular, the winner $\theta_{\ihat}$ is covered by $\widehat C^\alpha_{\ihat}$, the projection of $\widehat C^\alpha$ along the winning coordinate:
$$P(\theta_{\ihat} \in \widehat C^\alpha_{\ihat}) \geq P(\theta \in \widehat C^\alpha) \geq 1-\alpha, \text{ where } \widehat C^\alpha_{i} = \left\{\theta_{i}:\; \theta \in \widehat C^\alpha \right\} = \left\{\theta_{i}:\; \X \in A_\alpha(\theta)\right\}.$$

At first blush, it may seem like finding $\widehat C_{\ihat}^\alpha$ would require searching through all $\theta\in \R^m$ and checking if $\theta\in \widehat C^\alpha$. Our key observation is that there exists a computational shortcut: to check whether a given value $t\in\R$ is in $\widehat C_{\ihat}^\alpha$, it suffices to look at the ``worst-case'' vector $\theta^t$, defined by
\begin{equation}
\label{eq:worst-case-theta}
\theta_j^t = \begin{cases} t & \text{if } j = \ihat; \\ \min\left\{\frac{2}{3}\X_j + \frac{1}{3}t, t \right\} & \text{if } j \neq \ihat.\end{cases}
\end{equation}

The idea behind $\theta^t$ is that it gives the ``most favorable'' conditions for $t$ to be included in the projected confidence interval $\widehat C_{\ihat}^\alpha$. That is, if {\em any} $\theta$ with $\theta_{\ihat} = t$ is included in the joint confidence region $\widehat C^\alpha$, then $\theta^t$ will be included. Consequently, we can verify whether $t$ is in our confidence interval by performing our zoom test at the parameter vector $\theta^t$.
Formally, we have the following key technical lemma.

\begin{restatable}{lemma}{mainlemma}
\label{lemma:mainlemma}
Fix $t\in\R$ and let $\theta^t$ be as in Eq. \eqref{eq:worst-case-theta}. Then, $t \in \widehat C_{\ihat}^\alpha$ if and only if $|\X_{\ihat}-t| \leq r_\alpha(\theta^t)$.
\end{restatable}

\begin{figure}[t]
\centering
\includegraphics[width=0.45\textwidth]{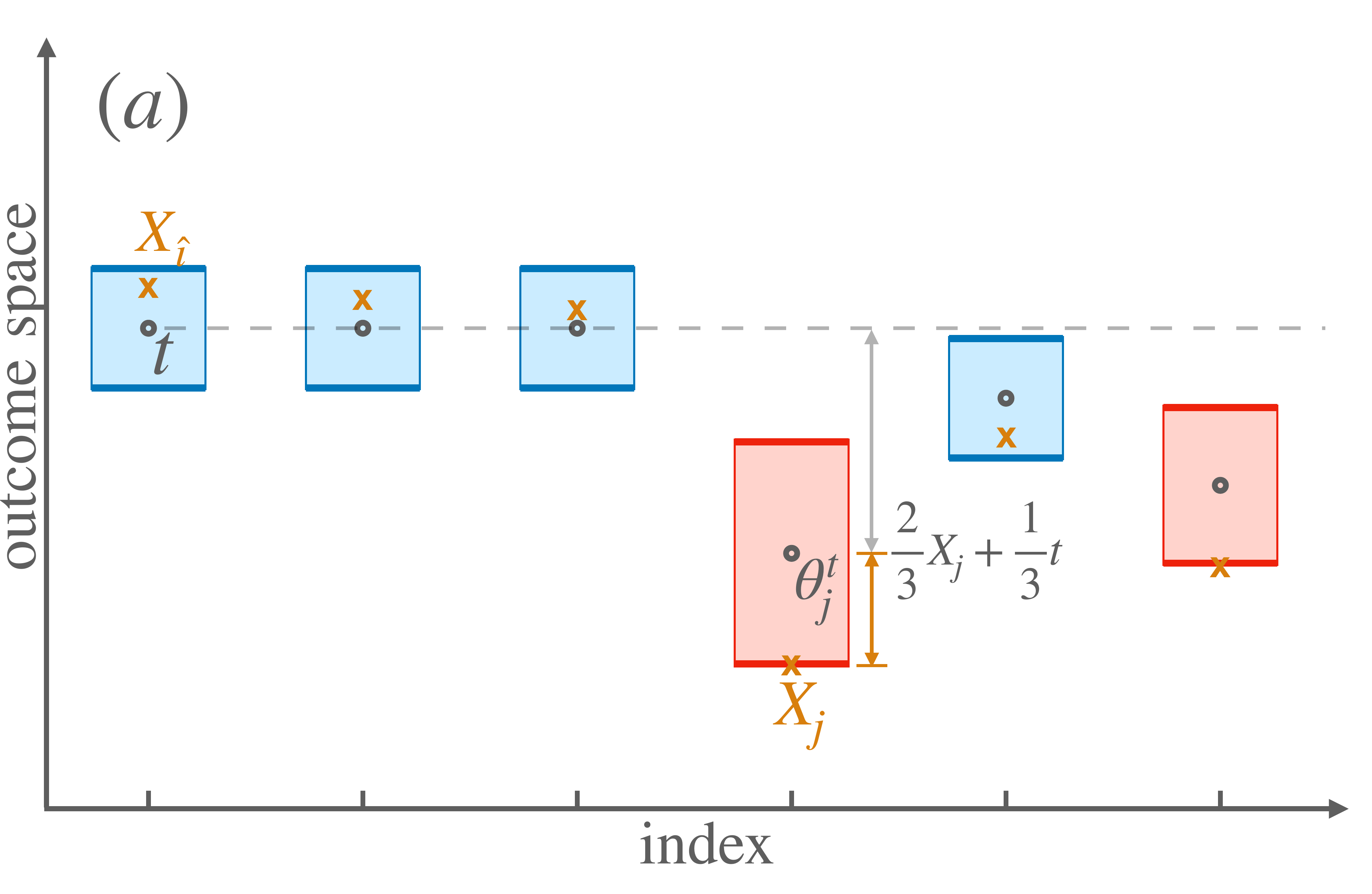}
\hspace{0.2cm}
\includegraphics[width=0.45\textwidth]{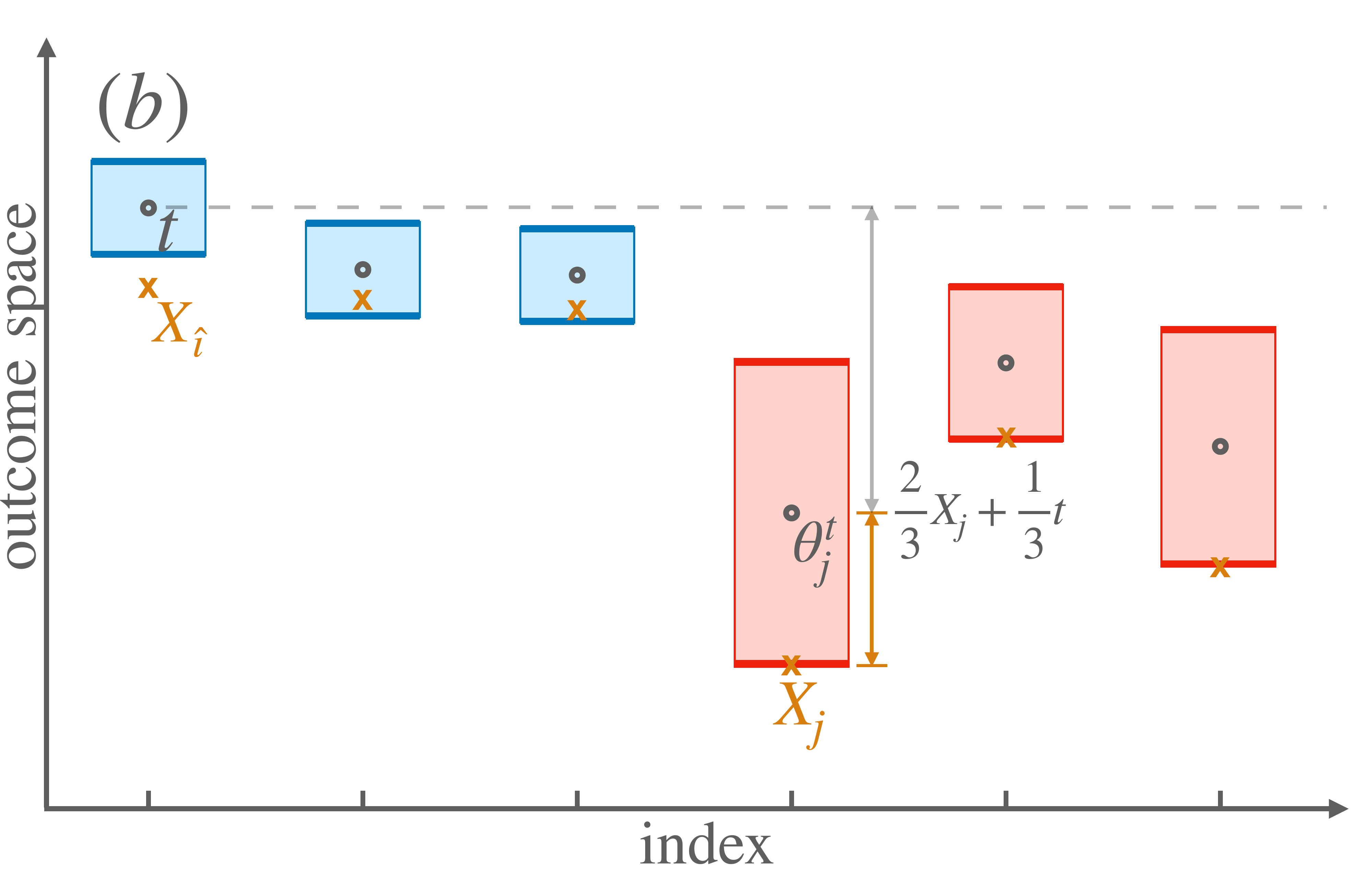}
\caption{\textbf{Data $X$ (brown crosses) and $\theta^t$ (gray circles) for two candidate values of $t$.} The intervals illustrate the acceptance region $A_\alpha(\theta^t)$. (a) We consider a value $t$ less than $X_{\ihat}$ and conclude that $t\in\widehat C^\alpha_{\ihat}$ because $X_{\ihat}$ is within the active radius $r_\alpha(\theta^t)$ around $t$. (b) We consider a value $t$ greater than $X_{\ihat}$ and conclude that $t\not\in\widehat C^\alpha_{\ihat}$ because $X_{\ihat}$ is \emph{not} within the active radius $r_\alpha(\theta^t)$ around $t$.}
\label{fig:theta_t}
\end{figure}

Figure \ref{fig:theta_t} depicts the data $X$ (brown crosses) and $\theta^t$ (gray circles) for two candidate values of $t$, one below $X_{\ihat}$ (left) and one above (right). The intervals illustrate the region $A_\alpha(\theta^t)$. We observe that, for $j\neq\ihat$, the values of $\theta_j^t$ are pushed as far up as possible, so that if $j$ becomes inactive $X_j$ is barely covered by the acceptance region. This is the reason why $\theta_j^t$ is the most favorable vector $\theta$ for $t$ to be included in $\widehat C_{\ihat}^\alpha$: the suboptimality gaps are as small as possible, making the active radius as large as possible, but they are not too small to violate the constraint that $X\in A_\alpha(\theta)$. Notice also that more coordinates are inactive for larger $t$, implying that the active radius $r_\alpha(\theta^t)$ will be smaller for larger $t$. This, in turn, implies that $\widehat C_{\ihat}^\alpha$ will be shorter on the upper than on the lower end of $X_{\ihat}$.

Lemma \ref{lemma:mainlemma} immediately implies the following confidence interval for the winner. 

\begin{restatable}[Zoom correction]{thm}{winner_int}
\label{thm:winner_int}
Let $\widehat C^\alpha_{\ihat} = \{t: |X_{\ihat}-t|\leq r_\alpha(\theta^t)\}$, and define  $t_l = \min \widehat C^\alpha_{\ihat}, t_u = \max \widehat C^\alpha_{\ihat}$. Then,
\[P(\theta_{\ihat} \in [t_l, t_u]) \geq 1-\alpha.\]
\end{restatable}

We make a few practical remarks. First, note that $\widehat C^\alpha_{\ihat}\subseteq [X_{\ihat}-r_\alpha(0), X_{\ihat} + r_\alpha(0)]$, where $r_\alpha(0)$ denotes the fully simultaneous interval radius, obtained when $\theta_1=\dots=\theta_m=0$. We do not have a guarantee that $\widehat C^\alpha_{\ihat}$ itself is an interval, though in practice we expect it to be; therefore, taking $ [t_l, t_u]$ should not make the confidence set much more conservative. In our experiments, we will implement the confidence interval prescribed by Theorem \ref{thm:winner_int} by performing a grid search over $t\in [X_{\ihat}-r_\alpha(0), X_{\ihat} + r_\alpha(0)]$. We note that this technically allows for some inaccuracy in the interval computation, but we expect this inaccuracy to be practically negligible.

In Section \ref{sec:extensions}, we will show that $\widehat C^\alpha_{\ihat}$ is also a valid confidence set for the \emph{population winner} $\theta^* = \max_{i\in[m]} \theta_i$. This will follow because, perhaps surprisingly, inverting the zoom test and projecting it along coordinate $i^* = \argmax_i \theta_i$ gives the same set as inverting the zoom test and projecting it along ${\ihat}$.

\section{Efficient step-down implementation}
\label{sec:stepdown}

We now provide a slightly conservative version of the approach in Theorem \ref{thm:winner_int} that gains in simplicity. The method can be seen as an ``effective'' Bonferroni correction: it performs a Bonferroni correction only over competitive candidates and thus adapts to the effective multiplicity of the problem.

We make no assumptions about the dependence of the errors $\xi_i$ and define $\mathbf{S}(v)$ via a union bound instead. Formally, for a vector $v$, we assume that $\mathbf{S}(v)$ takes the form:
\[\mathbf{S}(v) = \sum_{i=1}^m S(v_i),\]
where $S(r)$ denotes a decreasing tail upper bound on the marginal noise distribution: $P(|\xi_j|>r) \leq S(r)$ for all $j\in[m]$.
With this choice of $\mathbf{S}(v)$, we can give a simple characterization of $t_l$ and $t_u$, the lower and upper endpoints of $\widehat C_{\ihat}^\alpha$. Let $r_l = \X_{\ihat} - t_l$ and $r_u = t_u - \X_{\ihat}$ denote the radius of the lower and upper endpoints, respectively; in other words, $[t_l, t_u] = [\X_{\ihat} - r_l, \X_{\ihat} + r_u]$. We also let $\hDelta_j = X_{\ihat} - X_j$ denote the empirical suboptimality gaps.

\begin{restatable}{lemma}{endpointequation}
\label{lemma:endpoint_equation}
The radius $r_l$ of the lower endpoint of $\widehat C_{\ihat}^\alpha$ solves
\begin{equation}
\label{eq:lower}
\alpha = \S_l(r) := \sum_{j=1}^m S\left(\max\left\{r, \frac{\hDelta_j - r}{3}\right\}\right).
\end{equation}
The radius $r_u$ of the upper endpoint of $\widehat C_{\ihat}^\alpha$ solves
\begin{equation}
\label{eq:upper}
\alpha = \S_u(r) := \sum_{j=1}^m S\left(\max\left\{r, \frac{\hDelta_j + r}{3}\right\}\right).
\end{equation}
\end{restatable}

We immediately see that $\S_l(r) \geq \S_u(r)$ for all $r$, which means that the lower endpoint will offer more protection than the upper endpoint. This makes sense, since the winner is biased upwards. Moreover, both functions approach $S(r)$ from above as $\hDelta_j\to\infty$ for all $j\neq \ihat$.
Observing that $\hDelta_{\ihat}=0$, we also see that $\S_l(0) \geq \S_u(0) \geq S(0) = 1$ and $\S_l(r), \S_u(r)\to 0$ as $r\to\infty$, so the equations are guaranteed to have at least one solution for some $r \in (0,\infty)$. In fact, the function $\S_u(r)$ is decreasing and therefore has a unique solution to Eq.~\eqref{eq:upper}. The function $\S_l(r)$ is not necessarily decreasing, but it is bounded above by the decreasing function $m\cdot S(r)$.

Finally, the explicit characterization of Lemma \ref{lemma:endpoint_equation} allows for a straightforward comparison of our approach to a Bonferroni correction. The lower radius $r_l$ reduces to the Bonferroni radius $r_{\text{Bonf}}$, which solves $\alpha = \sum_{j=1}^m S(r_{\text{Bonf}})$, if and only if $\max_j\hDelta_j \leq 4r_{\text{Bonf}}$. Similarly, the upper radius $r_u$ reduces to the Bonferroni radius if and only if $\max_j\hDelta_j \leq 2r_{\text{Bonf}}$. As a result, our interval coincides with the Bonferroni interval if and only if the Bonferroni intervals for the winner and the loser overlap.

\paragraph{Step-down zoom correction.} The preceding discussion characterizes the endpoints of $\widehat C^\alpha_{\ihat}$ when $\mathbf{S}(v)$ is given by a union bound. By bounding $\S_l$ and $\S_u$, we can obtain a very efficient \emph{step-down zoom correction} that produces only slightly conservative upper bounds on $r_l$ and $r_u$. We state the estimation of $r_l$ and $r_u$ in Algorithm \ref{alg:stepdownlow} and Algorithm \ref{alg:stepdownhigh}, respectively.
\begin{figure}[t]
\centering
\begin{minipage}{.485\textwidth}
\begin{algorithm2e}[H]
    \KwData{
    $X_1,\dots,X_m$
    }
    \SetAlgoLined
    ~Compute gaps $\hDelta_j = X_{\ihat} - X_j,j\in[m]$, and sort $\hDelta_{(1)} \geq \hDelta_{(2)} \geq \cdots \geq \hDelta_{(m)} = 0$\newline
    $\alpha_1 = \alpha$\newline
    \For{$j = 1$ \KwTo $m$}{
        $\hat{r}_l^{(j)} = S^{-1}(\alpha_j / (m-j+1))$\newline
        \eIf{$\hDelta_{(j)} \leq 4 \hat{r}_l^{(j)}$}{
            \KwRet $\hat r_l = \hat{r}_l^{(j)}$
        }{
            $\alpha_{j+1} = \alpha_j - S\left(\frac{\hDelta_{(j)} - \hat{r}_l^{(j)}}{3}\right)$
        }
    } 
\caption{\parbox[t]{.9\linewidth}{Step-down zoom correction ($r_l$)}}
\label{alg:stepdownlow}
\end{algorithm2e}
\end{minipage}
\hfill
\begin{minipage}{.485\textwidth}
\begin{algorithm2e}[H]
    \KwData{
    $X_1,\dots,X_m$
    }
    \SetAlgoLined
    ~Compute gaps $\hDelta_j = X_{\ihat} - X_j,j\in[m]$, and sort $\hDelta_{(1)} \geq \hDelta_{(2)} \geq \cdots \geq \hDelta_{(m)} = 0$\newline
    $\alpha_1 = \alpha$\newline
    \For{$j = 1$ \KwTo $m$}{
        $\hat{r}_u^{(j)} = S^{-1}(\alpha_j / (m-j+1))$\newline
        \eIf{$\hDelta_{(j)} \leq 2 \hat{r}_u^{(j)}$}{
            \KwRet $\hat r_u = \hat{r}_u^{(j)}$
        }{
            $\alpha_{j+1} = \alpha_{j} - S\left(\frac{\hDelta_{(j)} + S^{-1}(\alpha)}{3}\right)$
        }
    }
\caption{\parbox[t]{.9\linewidth}{Step-down zoom correction ($r_u$)}}
\label{alg:stepdownhigh}
\end{algorithm2e}
\end{minipage}
\end{figure}
\noindent Note that $S^{-1}(\alpha)$, used in the computation of the upper bound, is the uncorrected interval width.

The returned estimates $\hat r_l$ and $\hat r_u$ are indeed upper bounds on the true values $r_l$ and $r_u$.

\begin{restatable}[Step-down zoom correction]{thm}{stepdown}
\label{thm:stepdown}
Let $\hat r_l$ and $\hat r_u$ denote the lower and upper bound estimates from  Algorithm~\ref{alg:stepdownlow} and Algorithm~\ref{alg:stepdownhigh}, respectively. Then, $\hat r_l \geq r_l$ and $\hat r_u \geq r_u$. Consequently,
\[P(\theta_{\ihat} \in [\X_{\ihat} - \hat r_l, \X_{\ihat} + \hat r_u])\geq 1-\alpha.\]
\end{restatable}
The proof of Theorem \ref{thm:stepdown} relies on showing that $\S_l(\hat r_l)\leq \alpha$ and $\S_u(\hat r_u)\leq \alpha$. Since $\S_u(r)$ is monotonically decreasing, this is enough to conclude that the upper estimate is valid: $\hat r_u \geq r_u$. Since $\S_l(r)$ may not be monotone, we additionally argue that $\S_l(r)<\alpha$ for all $r> \hat r_l$.

\section{Extensions}
\label{sec:extensions}

We generalize and extend our main results for inference on the winner to several related problems: inference on the top $k$ winners, inference on the value $\theta^*$ and index $i^*$ of the population winner, and inference on ``near-winners.'' We also provide a variance-adaptive version of the algorithm, which may be more appropriate when the candidates have different variances, or more generally when their tail bounds differ.

\subsection{Inference on top \texorpdfstring{$k$}{k} winners}
\label{sec:topk}

We study the problem of inference on the \emph{top $k$} winners: $\ihat_{(1)},\dots,\ihat_{(k)}$, corresponding to the $k$ largest values of $X$. The following analysis recovers the previous analysis for $k=1$, thus providing a strict generalization.

\paragraph{Zoom test (top $k$).} The key difference is in redefining the suboptimality gaps in the definition of the zoom test. Specifically, we take
\[\Delta^{(k)}_j = \max\{\theta_{(k)} - \theta_j , 0 \},\]
where $\theta_{(k)}$ denotes the $k$-th largest coordinate of $\theta$. Notice that this recovers the previous definition of $\Delta_j$ when $k=1$. The remaining definitions---acceptance region $A_\alpha$, active radius $r_\alpha$, and so on---all remain the same, only replacing the gaps $\Delta_j$ with the more general $\Delta_j^{(k)}$. For example, the active indices for general $k$ are given by $\I_\alpha = \{j:\; \Delta_j^{(k)} \leq 2r_\alpha\}.$

\paragraph{Inverting the test (top $k$).} 
As before, we will form a confidence region for coordinates $\ihat_{(1)},\dots,\ihat_{(k)}$ by first inverting $A_\alpha$ to obtain a confidence region for $\theta$ and then projecting it onto those coordinates:
\[
\widehat C^\alpha_{\ihat_{(1)},\dots,\ihat_{(k)}} = \left\{(\theta_{\ihat_{(1)}},\dots,\theta_{\ihat_{(k)}}):\; \theta \in \widehat C^\alpha \right\} = \left\{(\theta_{\ihat_{(1)}},\dots,\theta_{\ihat_{(k)}}) :\; \X \in A_\alpha(\theta)\right\}.
\]

The computational shortcut from the $k=1$ case has its direct analogue in the general case. In particular, for $\tvec = (t_{\ihat_{(1)}},\dots,t_{\ihat_{(k)}})\in\R^k$, we define the ``worst-case'' vector $\theta^{\tvec}$ as
\begin{equation}
\label{eq:worst-case-theta-topk}
\theta_j^{\tvec} = \begin{cases} t_j & \text{if } j \in \{\ihat_{(1)},\dots,\ihat_{(k)}\}; \\ \min\left\{\frac{2}{3}\X_j + \frac{1}{3}\tvec_{(k)}, \tvec_{(k)} \right\} & \text{if } j \not\in \{\ihat_{(1)},\dots,\ihat_{(k)}\},\end{cases}
\end{equation}
where $\tvec_{(k)} = \min_{j\in[k]}t_{\ihat_{(j)}}$.

Lemma \ref{lemma:main-lemma-topk} extends Lemma \ref{lemma:mainlemma} to general $k$ and is proved analogously. 

\begin{restatable}{lemma}{mainlemmatopk}
\label{lemma:main-lemma-topk}
Fix $\tvec=(t_{\ihat_{(1)}},\dots,t_{\ihat_{(k)}})\in\R^k$ and let $\theta^{\tvec}$ be as in Eq. \eqref{eq:worst-case-theta-topk}. Then, $\tvec \in \widehat C_{\ihat_{(1)},\dots, \ihat_{(k)}}^\alpha$  if and only if $\max_{j\in[k]}|X_{\ihat_{(j)}} - t_{\ihat_{(j)}}| \leq r_\alpha(\theta^{\tvec})$.
\end{restatable}

One issue with Lemma \ref{lemma:main-lemma-topk} is that for non-trivially large $k$ the grid search approach described at the end of Section \ref{sec:conf_int_winner} is not feasible: instead of searching over a one-dimensional grid as before, we would need to search over a $k$-dimensional grid. To mitigate this issue, we show a reduction of Lemma \ref{lemma:main-lemma-topk} to a one-dimensional problem. The following lemma no longer gives an exact characterization of points in the projection set $\widehat C^\alpha_{\ihat_{(1)},\dots,\ihat_{(k)}}$, but it gives a sufficient condition that allows forming a valid confidence set.

\begin{restatable}{lemma}{topkoned}
\label{lemma:topk-1d}
Fix $\rvec = (r_{\ihat_{(1)}},\dots,r_{\ihat_{(k)}})\in\R^k$.
Given $r\in\R$, define
\begin{equation}
\label{eq:thetatil}
\tilde\theta_j^r = \begin{cases} \X_j-r & \text{if } j \in \{\ihat_{(1)},\dots,\ihat_{(k)}\}; \\ \min\left\{\frac{2}{3}\X_j + \frac{1}{3}(\X_{\ihat_{(k)}} - r), \X_{\ihat_{(k)}} - r \right\} & \text{if } j \not\in \{\ihat_{(1)},\dots,\ihat_{(k)}\}.\end{cases}
\end{equation}
If $(\X_{\ihat_{(1)}} - r_{\ihat_{(1)}},\dots, \X_{\ihat_{(k)}} - r_{\ihat_{(k)}}) \in \widehat C_{\ihat_{(1)},\dots, \ihat_{(k)}}$, then $|r_{\ihat_{(j_0)}}| \leq r_\alpha\left(\tilde \theta^{|r_{\ihat_{(j_0)}}|}\right)$, where $j_0 = \argmax_{j\in[k]} |r_{\ihat_{(j)}}|$.
\end{restatable}

\noindent Lemma \ref{lemma:topk-1d} implies a one-dimensional approach to inference on the top $k$ winners, similarly to Theorem \ref{thm:winner_int}.

\begin{restatable}[Zoom correction: top $k$]{thm}{topk}
\label{thm:winner-topk}
Let $r_{\max} = \max\{r: r\leq r_\alpha(\tilde\theta^{r}) \}$, where $\tilde\theta^{r}$ is as in Eq.~\eqref{eq:thetatil}. Then,
\[P\left(\theta_{\ihat_{(1)},\dots,\ihat_{(k)}} \in [X_{\ihat_{(1)},\dots,\ihat_{(k)}} \pm r_{\max}] \right)\geq 1-\alpha.\]
\end{restatable}
\noindent As in the case of Theorem \ref{thm:winner_int}, the interval width can be no larger than the fully simultanteous correction, $r_{\max}\in[-r_\alpha(0), r_\alpha(0)]$, obtained when $\theta_1=\dots=\theta_m=0$. This means that in practice we can find $r_{\max}$ by performing a grid search over $r\in [-r_\alpha(0), r_\alpha(0)]$. Of course, grid search allows for numerical inaccuracies, however we do not expect such inaccuracies to affect the inferential validity.

Next, we discuss an analogue of the step-down correction from Section \ref{sec:stepdown}. Again, we assume that $\mathbf{S}(v)$ takes a union bound: $\mathbf{S}(v) = \sum_{i=1}^m S(v_i)$ for $v\in\R_+^m$. We let $\hDelta_j^{(k)} = X_{\ihat_{(k)}} - X_j$ denote the empirical suboptimality gaps for the top $k$ problem. It is not hard to show that $r_{\max}$ solves the \emph{same} equation as $r_l$ in Lemma \ref{lemma:endpoint_equation}, only replacing $\hDelta_j$ with $\hDelta_j^{(k)}$.

\begin{restatable}{lemma}{endpointeqtopk}
\label{lemma:endpoint_equation_topk}
The radius $r_{\max}$ solves
$\alpha = \sum_{j=1}^m S\left(\max\left\{r, \frac{\hDelta_j^{(k)} - r}{3}\right\}\right).$
\end{restatable}

This means that the step-down method in Algorithm \ref{alg:stepdownlow} can be immediately repurposed for inference on the top $k$ winners. In particular, running Algorithm \ref{alg:stepdownlow} with the suboptimality gaps $\hDelta_j^{(k)}$ yields an estimate $\hat r_{\max}$ that is guaranteed to upper bound $r_{\max}$.

\begin{restatable}[Step-down zoom correction: top $k$]{thm}{stepdowntopk}
\label{thm:stepdowntopk}
Let $\hat r_{\max}$ denote the output of  Algorithm~\ref{alg:stepdownlow} run with $\hDelta_j^{(k)}$ instead of $\hDelta_j$. Then, $\hat r_{\max}\geq r_{\max}$. Consequently,
\[P\left(\theta_{\ihat_{(1)},\dots,\ihat_{(k)}} \in [X_{\ihat_{(1)},\dots,\ihat_{(k)}} \pm \hat r_{\max}]\right)\geq 1-\alpha.\]
\end{restatable}

\subsection{Inference on the population winner}
\label{sec:population_winner}

Next, we study inference on the value and index of the \emph{population winner}:
$$\theta^* = \max_{i\in[m]} \theta_i \quad\text{ and }\quad  i^* = \argmax_{i\in[m]} \theta_i.$$
In both cases, we form a confidence set as before: we invert $A_\alpha$ and apply a projection.

\paragraph{Inference on the value.} First, we show that $\widehat C^\alpha_{\ihat}$, used in Section \ref{sec:conf_int_winner}, is also a valid confidence set for $\theta^*$ under only mild assumptions.
The claim follows from the fact that $\{ \theta^* : X \in A_\alpha(\theta)\} = \{\theta_{\ihat} : X \in A_\alpha(\theta)\}$. In other words, inverting the test given by $A_\alpha$ and projecting it along coordinate ${i^*}$ is equivalent to inverting the same test and projecting it along ${\ihat}$.

We formalize this in Proposition~\ref{prop:population_winner}. We say that $\mathbf{S}(v)$ is symmetric in the entries of $v$ if, for any permutation $\sigma$, $\mathbf{S}(v) = \mathbf{S}(v_{\sigma})$, where $v_{\sigma}$ is obtained by permuting the entries of $v$ according to~$\sigma$.

\begin{restatable}{prop}{populationwinner}
    \label{prop:population_winner}
Suppose that $\mathbf{S}(v)$ is symmetric in the entries of $v$. Let 
$$\widehat C^{\alpha}_{*} = \{ \theta^* : X \in A_\alpha(\theta)\} \text{ and } \widehat C^{\alpha}_{\ihat} = \{ \theta_{\ihat} : X \in A_\alpha(\theta)\}.$$
Then, $\widehat C^{\alpha}_{*} = \widehat C_{\ihat}^\alpha$. Consequently, $P(\theta^* \in \widehat C^\alpha_{\ihat}) \geq 1-\alpha$.
\end{restatable}

If $\mathbf{S}$ relies on a union bound, as in Section \ref{sec:stepdown}, then it is symmetric in the entries of $v$. Therefore, the step-down zoom correction also applies to the population winner $\theta^*$.

\paragraph{Inference on the index.} Often we are interested in the identity of the population winner, and not necessarily its value. For example, we may be interested in knowing which treatment is the most effective.
If $\theta$ has multiple population winners, meaning $\theta^* = \max_i \theta_i$ is attained at multiple coordinates, then we want a confidence set that includes all of them.

As before, we form this set by inverting the hypothesis test $A_\alpha$:
\begin{equation}
\label{eq:winning_index_set}
\widehat \I^\alpha = \left\{i\in[m]: X\in A_\alpha(\theta), \theta_i = \max_{j\in[m]} \theta_j \right\}.
\end{equation}
The ``worst-case'' vector $\theta^t$ \eqref{eq:worst-case-theta} again provides a simple characterization of the confidence set $\widehat \I^\alpha$. Recall that $t_l = \min \widehat C^\alpha_{\ihat}$ denotes the lower endpoint of the confidence interval for the winner.

\begin{restatable}{prop}{inferenceindex}
\label{prop:inference-on-index}
We have $\widehat \I^\alpha = \{i : X_i \geq X_{\ihat} - 2 r_\alpha(\theta^{t_l})\}$. Consequently, $P(X_{i^*} \geq X_{\ihat} - 2 r_\alpha(\theta^{t_l}))\geq 1-\alpha$.
\end{restatable}

Therefore, we can simply collect the indices that are at most $2 r_\alpha(\theta^{t_l})$ away from being the empirical winner, and that set will be guaranteed to contain the population winner with high probability.

\subsection{Inference on near-winners}
\label{sec:near-winners}

In Section \ref{sec:conf_int_winner}, we projected the confidence set $\widehat C^\alpha$, obtained by inverting the zoom test, along the winning coordinate: $\widehat C^\alpha_{\ihat} = \{\theta_{\ihat} : \theta\in \widehat C^\alpha\}$. However, after seeing the data one might decide that other indices are also of interest, not just the winner. This might happen if, for example, several candidates are close to being tied with the winner.

The projection approach is in principle valid for \emph{any} coordinate, not just the winner. That is, for any index $j$ chosen in a possibly data-dependent manner,
$\widehat C_j^\alpha = \{\theta_j : \theta \in \widehat C^\alpha\}$
is a valid $(1-\alpha)$-confidence set. In Section \ref{sec:conf_int_winner} we showed how to compute this set exactly when $j$ is the winner, $j = \ihat$. In this section, we show how to obtain a conservative approximation of this set for \emph{any} index $j$. As in Proposition~\ref{prop:population_winner}, we rely on the mild assumption that $\mathbf{S}(v)$ is symmetric in the entries of $v$, meaning that for any permutation $\sigma$, $\mathbf{S}(v) = \mathbf{S}(v_{\sigma})$, where $v_{\sigma}$ is the vector obtained by permuting the entries of $v$ according to $\sigma$. This assumption is true when applying a union bound, as in Section~\ref{sec:stepdown}.

Recall that $r_l = X_{\ihat}-t_l$ and $r_u = t_u - X_{\ihat}$ are the lower and upper radius of the confidence interval for the winner, $[t_l, t_u]$.

\begin{restatable}{prop}{nearwinners}
\label{prop:near-winners}
Suppose that $\mathbf{S}(v)$ is symmetric in the entries of $v$. For any, possibly data-dependent $j\in[m]$, let
\[\widetilde C_j^\alpha = \Big(\left[X_{\ihat} - 3r_l, X_{\ihat} + r_u\right]\cap \left[X_j - r_l, X_j + r_l\right]\Big) \cup \left[X_j - \hDelta_j - r_u, X_j + \frac{1}{3}(\hDelta_j + r_u)\right].\]
Then, $\widehat C_j^\alpha \subseteq \widetilde C_j^\alpha$. Consequently, $P(\theta_j \in \widetilde C_{j}^\alpha) \geq 1-\alpha$.
\end{restatable}

\noindent Notice that for the winner, i.e. when $\hDelta_j = 0$, $\widetilde C^\alpha_j$ exactly coincides with the set from Theorem \ref{thm:winner_int}.

\subsection{Variance-adaptive algorithm}
\label{sec:variance-adaptive}

The main procedure from Section \ref{sec:conf_int_winner} computes the active radius $r_\alpha$ independently of any information about the variance of different candidates. However, if the winner $\ihat$ happens to have a small variance relative to the other candidates, this approach may lead to an unnecessarily wide interval for the winner. To address this issue, we extend our main results to allow variance-adaptive interval widths.

We assume that for every coordinate $i$ we have a known estimate of the variance of $X_i$, denoted $\sigma_i^2$. While we will use the term ``variance'' to refer to $\sigma_i^2$, note that $\sigma_i^2$ need not correspond to the actual variance; the procedure will be valid for any choice of $\sigma_1^2,\dots,\sigma_m^2$.

\paragraph{Variance-adaptive zoom test.} The main difference from the basic case is in the definition of the zoom test. We define $A_\alpha$ as the rectangle centered at $\theta$ whose radius in coordinate $j$ is $\max\left\{r_\alpha \sigma_j, \Delta_j \frac{\sigma_j}{\sigma_j + \sigma_{i^*}}\right\}$, where $i^* = \argmax_i \theta_i$ is the population winner:
\[
A_\alpha = \left[\theta \pm \left(r_\alpha \sigma\vee \Delta \frac{\sigma}{\sigma + \sigma_{i^*}}\right)\right].
\]
Therefore, the radius in coordinate $j$ is proportional to $\sigma_j$. As before, $r_\alpha$ is chosen to ensure that $A_\alpha$ is indeed a $(1-\alpha)$-level acceptance region:
\begin{equation}
\label{eq:active-radius-sigmas}
r_\alpha = \min \left\{r:\; \mathbf{S}\left(r \sigma \vee \Delta \frac{\sigma}{\sigma + \sigma_{i^*}} \right) \leq \alpha \right\}.
\end{equation}
The active indices $\I_\alpha$ are now those indices whose width is $r_\alpha \sigma_j$: $
\I_\alpha = \left\{j:\; \Delta_j \frac{1}{\sigma_j + \sigma_{i^*}} \leq r_\alpha \right\}.$

When $\sigma_1 = \dots = \sigma_m$, the variance-adaptive zoom test is equivalent to the basic test from Section~\ref{sec:test}.

\paragraph{Inverting the variance-adaptive test.} We again invert the test by relying on a ``worst-case'' vector, similar to \eqref{eq:worst-case-theta}. Given a tuple $(t,t^*,i^*) \in \R \times \R \times [m]$ such that $t^*\geq t$, we let:
$$\theta^{(t,t^*,i^*)}_j = \begin{cases}
t & \text{if } j = \ihat,\\
t^* & \text{if } j=i^*,\\
\min\left\{\frac{X_j(\sigma_j + \sigma_{i^*}) + t^* \sigma_j}{2\sigma_j + \sigma_{i^*}}, t^*\right\} & \text{if } j\not\in\{i^*,\ihat\}.
\end{cases}$$
We have the following key technical lemma, analogous to Lemma \ref{lemma:mainlemma}.

\begin{restatable}{lemma}{mainlemmasigmas}
\label{lemma:main-lemma-sigmas}
Fix $t\in\R$. Then, $t \in \widehat C_{\ihat}^\alpha$ if and only if, for some $t^*\geq t$ and $i^*\in[m]$, $|\X_{\ihat}-t| \leq r_\alpha(\theta^{(t,t^*,i^*)}) \sigma_{\ihat}$ and $|\X_j-t^*| \leq r_\alpha(\theta^{(t,t^*,i^*)}) \sigma_{j}$ for $j\in\{i^*\}\cup \{k: X_k \geq t^*\}$.
\end{restatable}

Similarly to Lemma \ref{lemma:mainlemma}, Lemma \ref{lemma:main-lemma-sigmas} reduces the problem of finding $\widehat C^\alpha_{\ihat}$ to a grid search problem. For every candidate value $t$, we additionally search through candidate population winners $i^*\in[m]$ and find the minimum value $t^*\geq t$ such that the condition of the lemma is satisfied. Once such a $t^*$ is found, we know that $t$ is in $\widehat C^\alpha_{\ihat}$ and we can move on to the next candidate value $t$.

\section{Experiments}

We evaluate the zoom correction on synthetic and real data, comparing it to both conditional solutions---in particular, standard conditional \cite{lee2016exact} and hybrid~\cite{andrews2019inference} inference---as well as simultaneous solutions---in particular, standard simultaneous inference, locally simultaneous inference \cite{zrnic2022locally}, and simultaneous inference over the selected (SoS) \cite{benjamini2019confidence}.
SoS intervals assume independent observations, so we compute them in settings where this assumption is satisfied. Conditional and hybrid inference are applicable to Gaussian observations with a known (or estimable) covariance, so we apply them in settings where we can compute a covariance estimate. Note that the intervals given by locally simultaneous inference are guaranteed to be no wider than standard, fully simultaneous intervals. We also implemented the intervals prescribed by \citet{fuentes2018confidence}, which are applicable in a subset of the settings where SoS intervals are applicable. However, we did not observe an advantage over SoS, and thus we will not plot these intervals.

We set the error level to be $\alpha=0.1$ throughout. Hybrid inference and locally simultaneous inference rely on an error budget splitting parameter $\beta\in(0,\alpha)$ (resp. $\nu\in(0,\alpha)$), which conceptually serve a similar purpose. We set them to be $\beta=\nu=0.1\alpha$ throughout. Similarly, SoS intervals have a parameter $\delta\in(0,1)$ that splits the error budget between the upper and lower endpoints of the interval; we choose the split to be balanced by setting $\delta=0.5$.

All code and data are available at: \url{https://github.com/tijana-zrnic/winners-curse}.

\subsection{Synthetic data}

\begin{figure}
\includegraphics[width=0.33\textwidth]{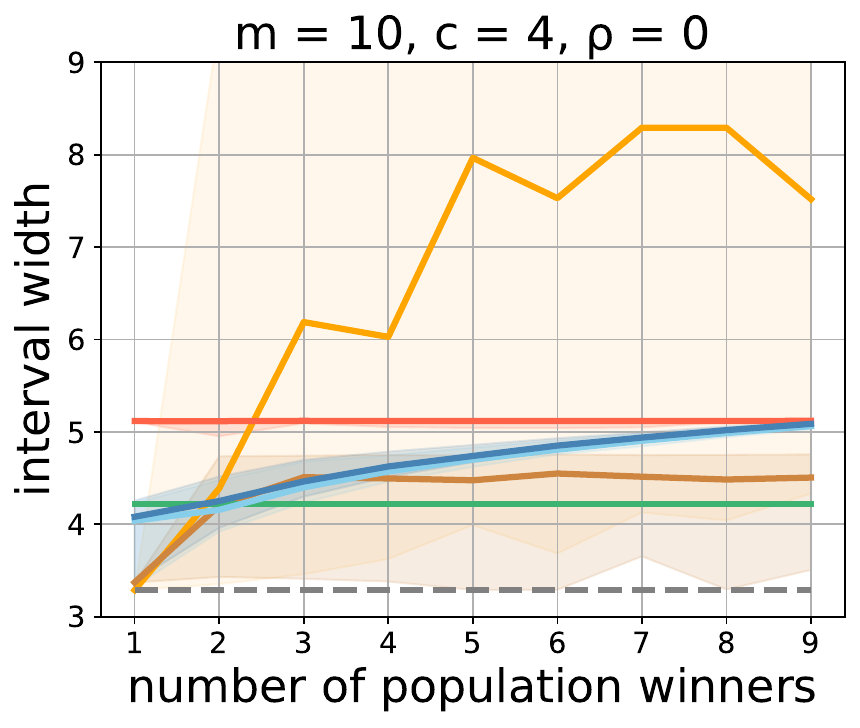}
\includegraphics[width=0.33\textwidth]{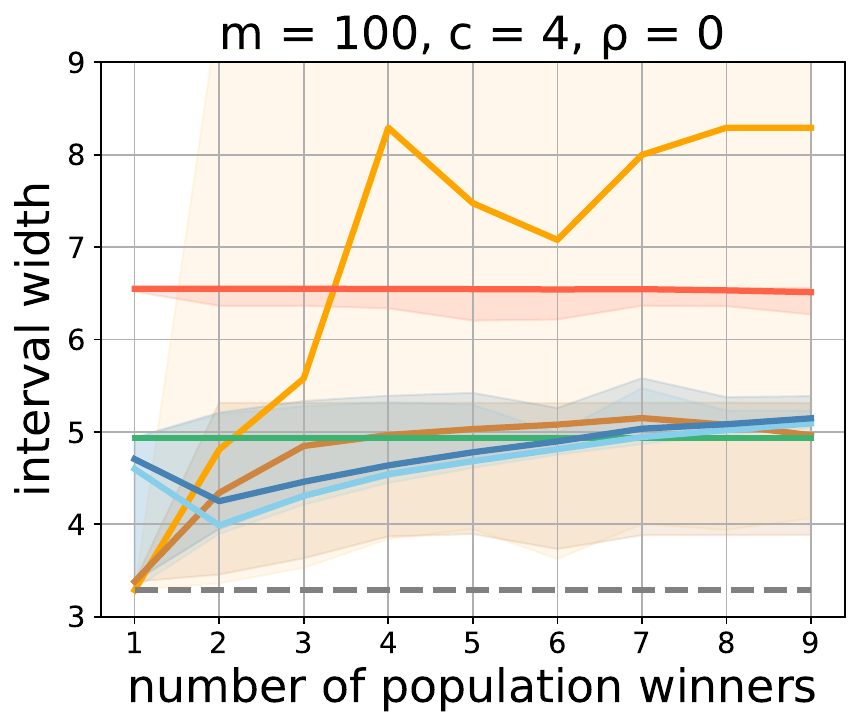}
\includegraphics[width=0.33\textwidth]{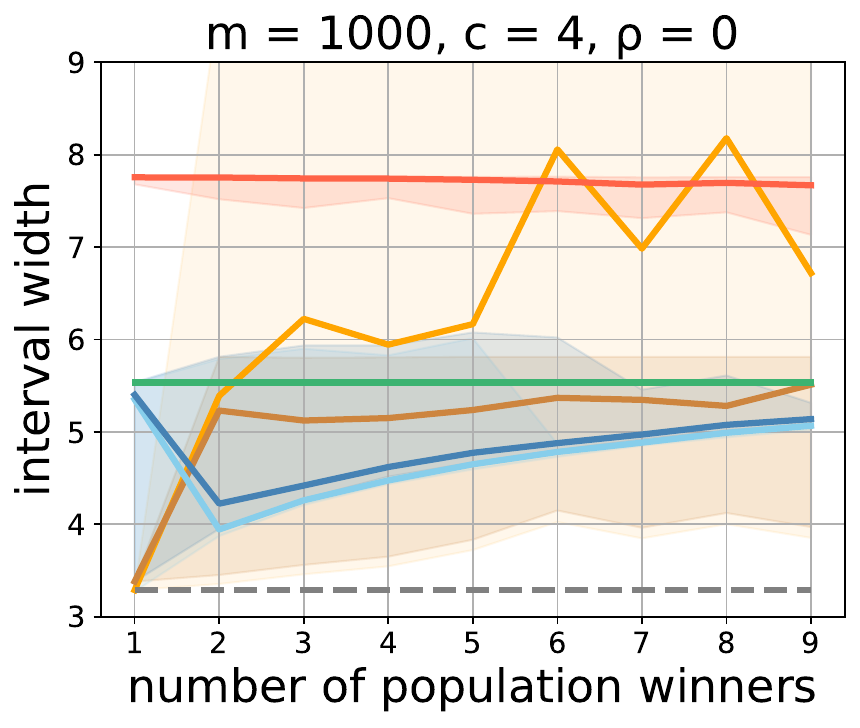}
\includegraphics[width=0.33\textwidth]{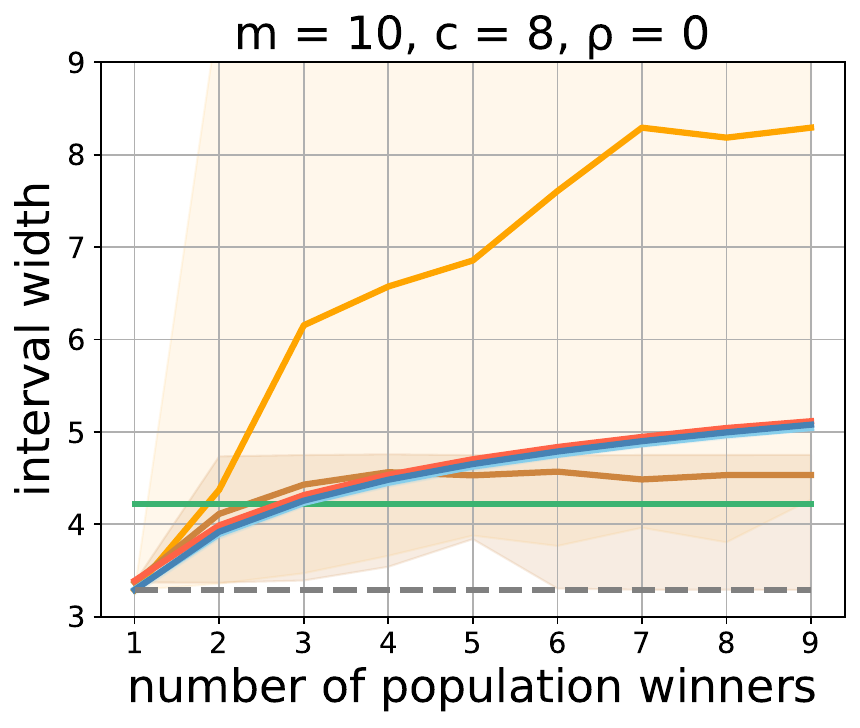}
\includegraphics[width=0.33\textwidth]{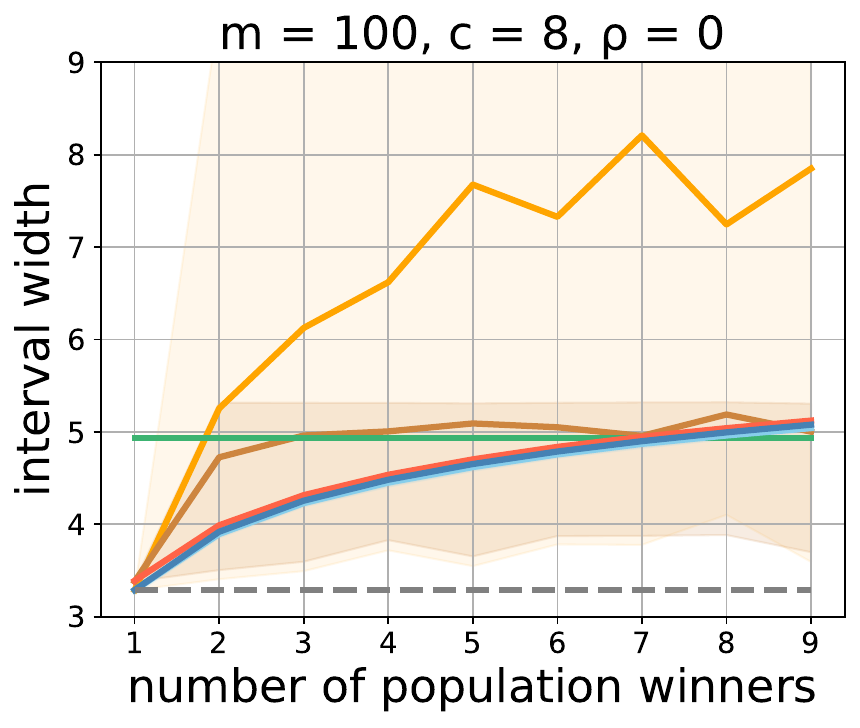}
\includegraphics[width=0.33\textwidth]{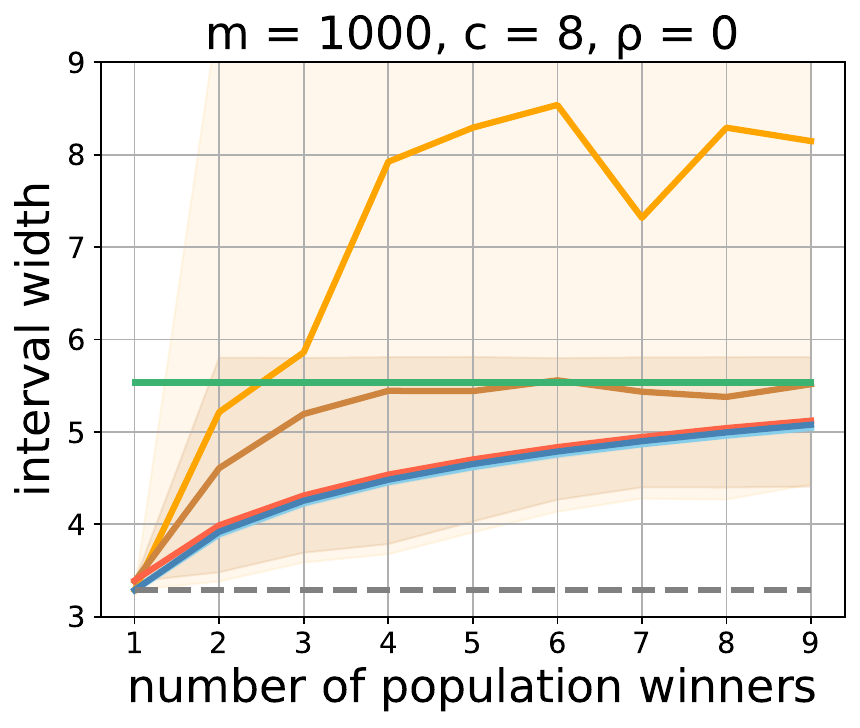}
\includegraphics[width=0.33\textwidth]{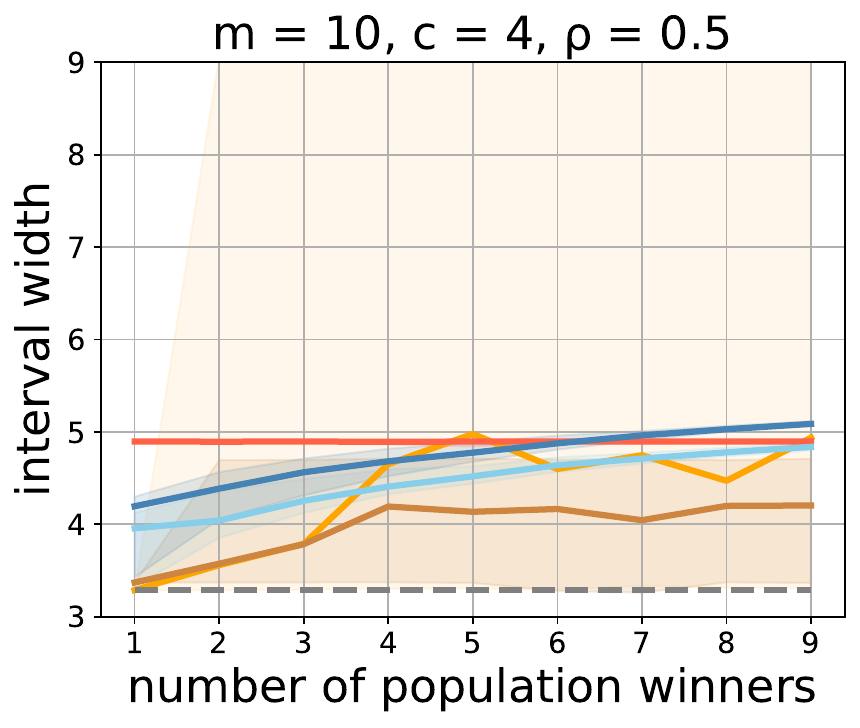}
\includegraphics[width=0.33\textwidth]{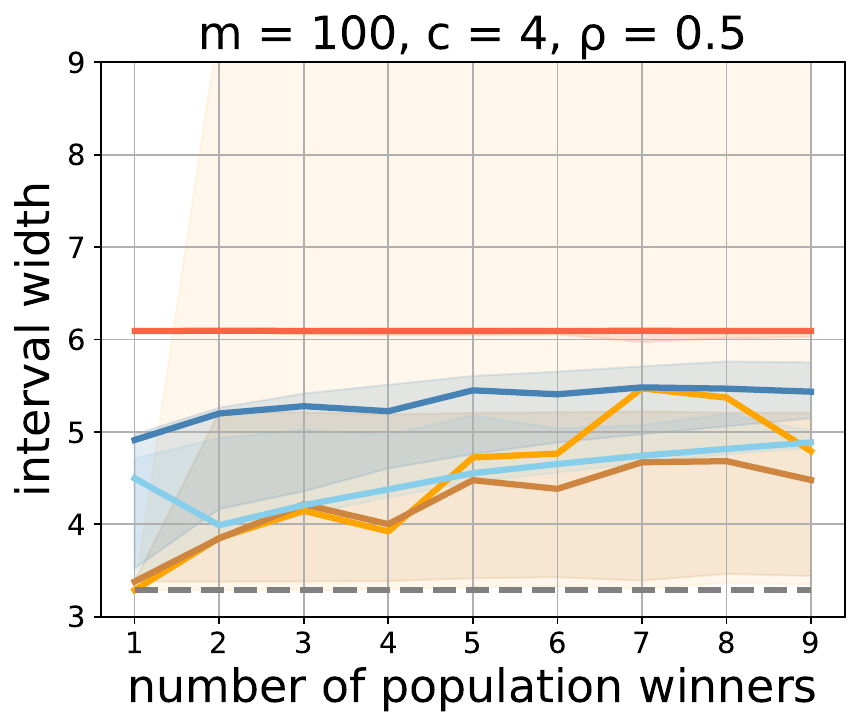}
\includegraphics[width=0.33\textwidth]{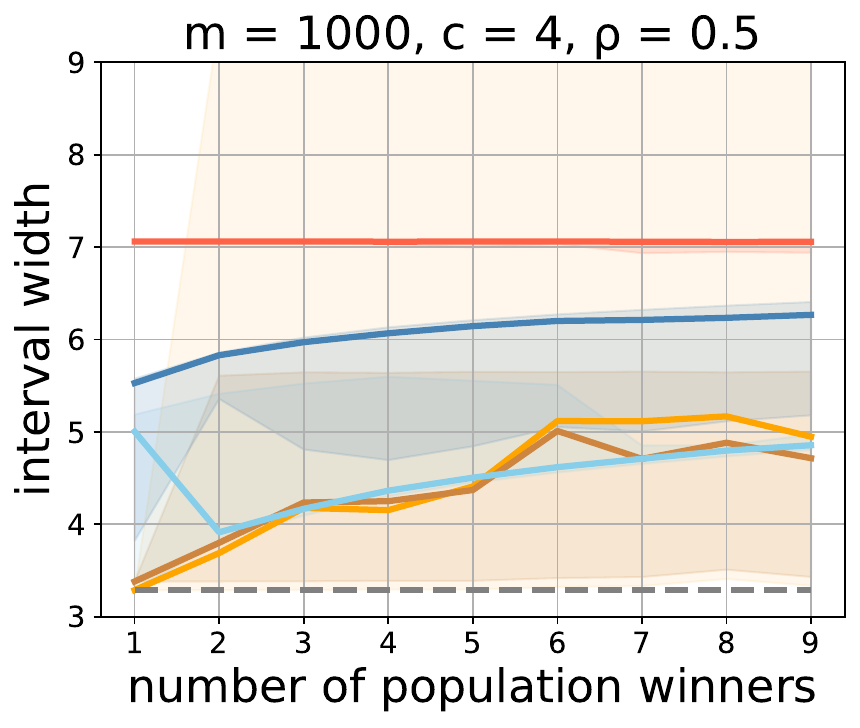}
\includegraphics[width=0.33\textwidth]{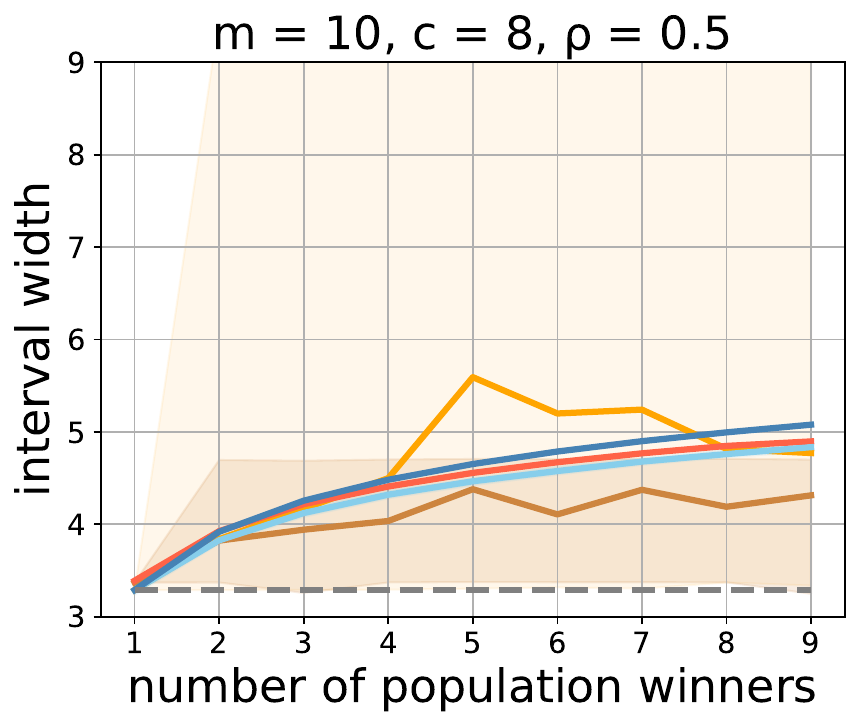}
\includegraphics[width=0.33\textwidth]{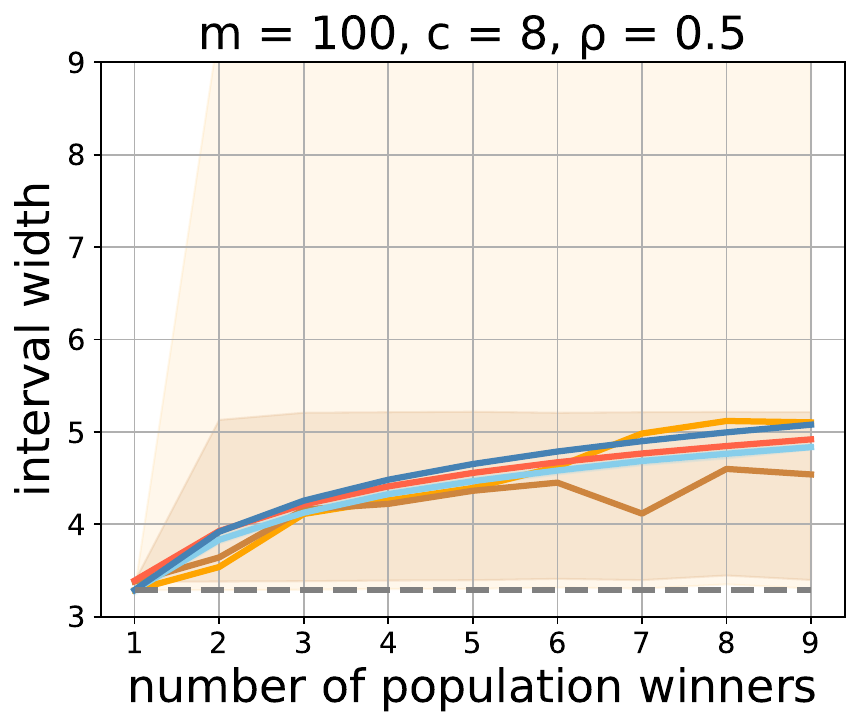}
\includegraphics[width=0.33\textwidth]{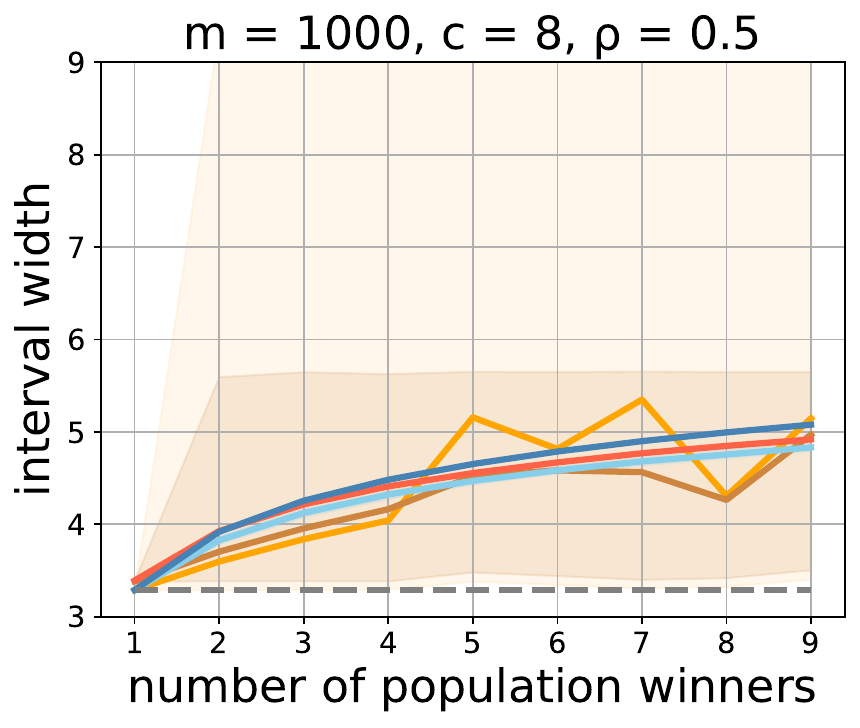}
\includegraphics[width=\textwidth]{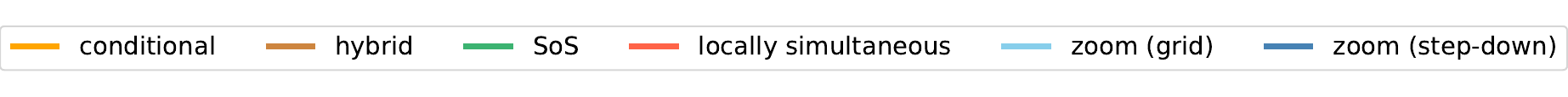}
\caption{Interval widths of the zoom correction compared to conditional \cite{lee2016exact}, hybrid \cite{andrews2019inference}, simultaneous over the selected (SoS) \cite{benjamini2019confidence}, and locally simultaneous inference \cite{zrnic2022locally}, in the synthetic data setting. The gray dashed line corresponds to the uncorrected interval width.}
\label{fig:synthetic}
\end{figure}

To build intuition, we begin with simulations on synthetic data. We report the median interval width over $100$ trials, together with the 5\% and 95\% percentile of interval width, plotted as error bars around the median. We consider two versions of the zoom method: one based on grid search, as explained in Section \ref{sec:conf_int_winner}, and the step-down version described in Section \ref{sec:stepdown}.

We sample $X \sim N(\theta, \Sigma)$, where $\theta$ is an $m$-dimensional mean vector, for varying $m$. We set
\[\theta_i = \begin{cases}
0, &i\in\{1,\dots,m_W \}\\
-c \cdot r_{\mathrm{sim}}, &i\in\{m_W +1,\dots,m\},
\end{cases}\]
where $m_W$ is a varying number of population winners, $c >0$ is a varying constant, and $r_{\mathrm{sim}}$ is the radius of the simultaneous (max-z) interval. The covariance matrix $\Sigma$ has $\Sigma_{ii} = 1$ and $\Sigma_{ij} = \rho$ for $i\neq j$. When $c=4$, the simultaneously corrected confidence interval for the winner and loser do not overlap with high probability. We are interested in the gap $4r_{\mathrm{sim}}$ because this roughly corresponds to the critical gap beyond which locally simultaneous inference drops the losing observations from its correction. This is the value for which locally simultaneous inference is most conservative, since it considers all observations despite them being fairly suboptimal.

In Figure \ref{fig:synthetic} we plot the interval widths for $c\in\{4, 8\}$, $\rho \in \{0,0.5\}$, and $m\in\{10,100,1000\}$. On the x-axis we vary the number of population winners $m_W$. The zoom correction is always at least as powerful as locally simultaneous inference, outperforming it by a large margin for smaller $c$. The two variants of the zoom correction, based on grid search and the step-down implementation, yield almost indistinguishable widths when the observations are independent, but when there are strong dependencies the grid-based approach is preferred. Conditional inference often yields very wide and highly variable intervals. Hybrid inference is generally comparable to the zoom correction, becoming more conservative when $m$ is large. SoS intervals are only applicable when $\rho=0$ and they are most conservative when $m_W$ is small, because they are not adaptive to the number of competitive candidates but only depend on the total number of candidates $m$.


\subsection{Extreme climate events}

We conduct experiments on the real climate dataset \cite{rasp2020weatherbench} considered in \cite{zrnic2022locally}. The dataset contains hourly measurements of temperature from 1999 to 2018 across a discrete grid of locations on Earth. The grid is obtained by pairing 32 latitude coordinates with 64 longitude coordinates. Following \cite{zrnic2022locally}, we model the measurements across years as i.i.d. multivariate Gaussian draws and use older data, from 1979 to 1998, to estimate the covariance. In the first set of experiments we compute the average temperature on Earth (averaged over all locations on the grid) and look at the resulting time series. For each year we take one measurement per day, evaluated at noon, resulting in a series of 365 entries. We ask for inference on the warmest day and the coldest day. In the second set of experiments we compute the average annual temperature and look at its distribution over the recorded locations on Earth. Similarly to the first set of experiments, we ask for inference on the warmest location and the coldest location.

In Figure \ref{fig:climate} we plot the intervals. The zoom correction consistently outperforms locally simultaneous inference (and thus also fully simultaneous inference). It is somewhat conservative in the problem of inference on the warmest location, but in all other problems it is comparable to the hybrid method. In the problems that focus on the most extreme days, conditional inference yields very large intervals, while in the problems that focus on the most extreme locations it yields the smallest intervals.

\begin{figure}[t]
\hspace{2.2mm}\includegraphics[height=0.3\textwidth]{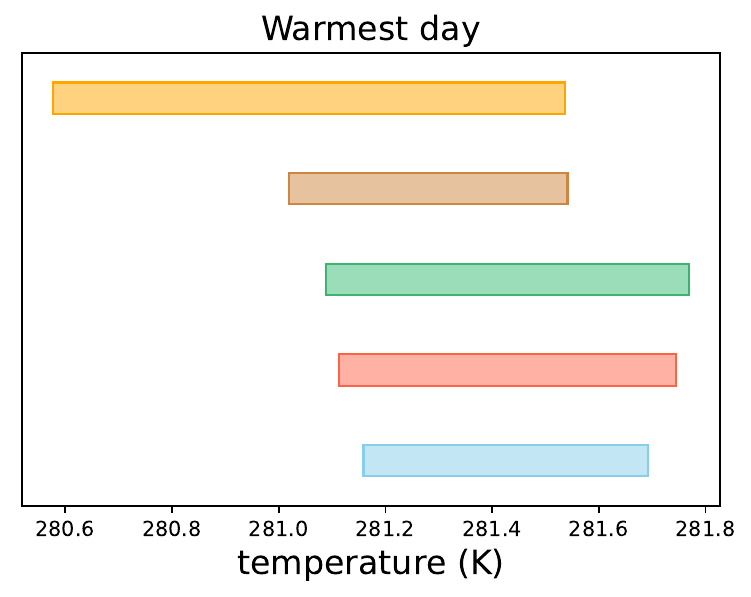}
\includegraphics[height=0.3\textwidth]{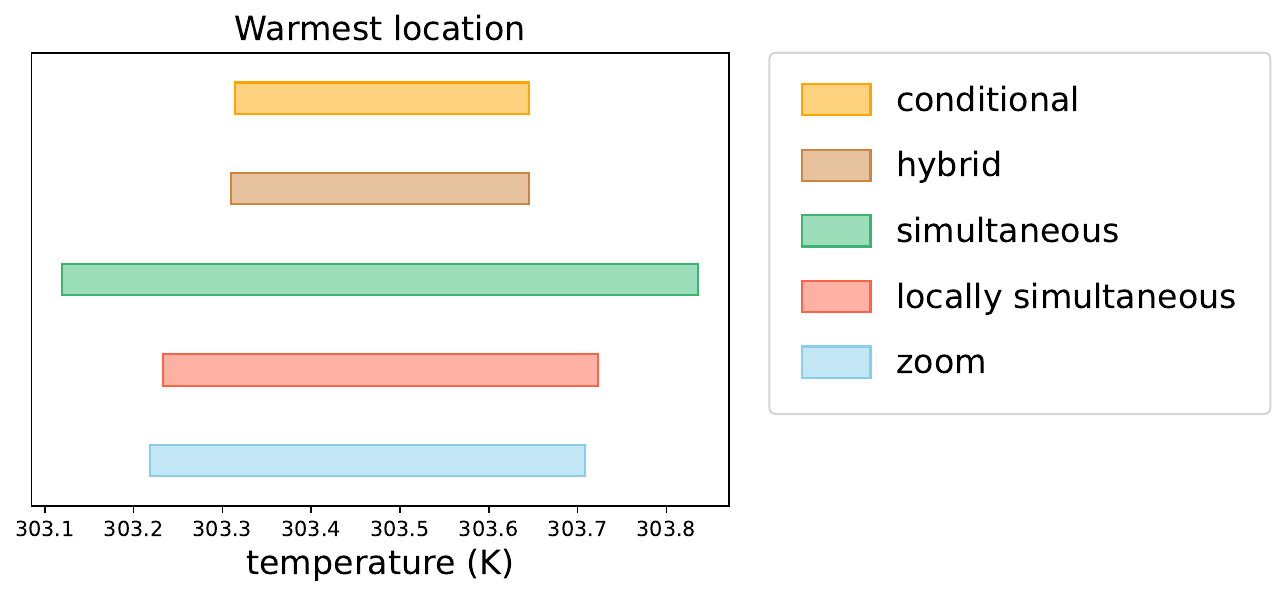}
\includegraphics[height=0.3\textwidth]{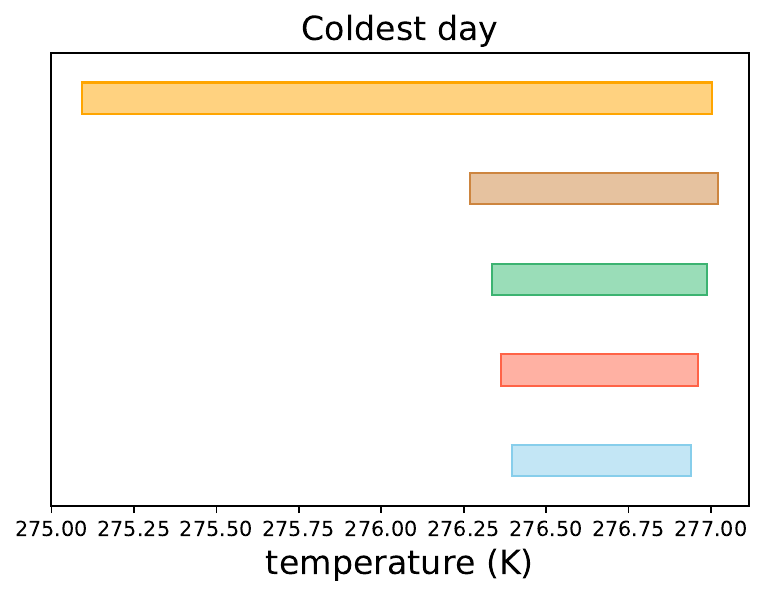}
\hspace{0.4mm}
\includegraphics[height=0.3\textwidth]{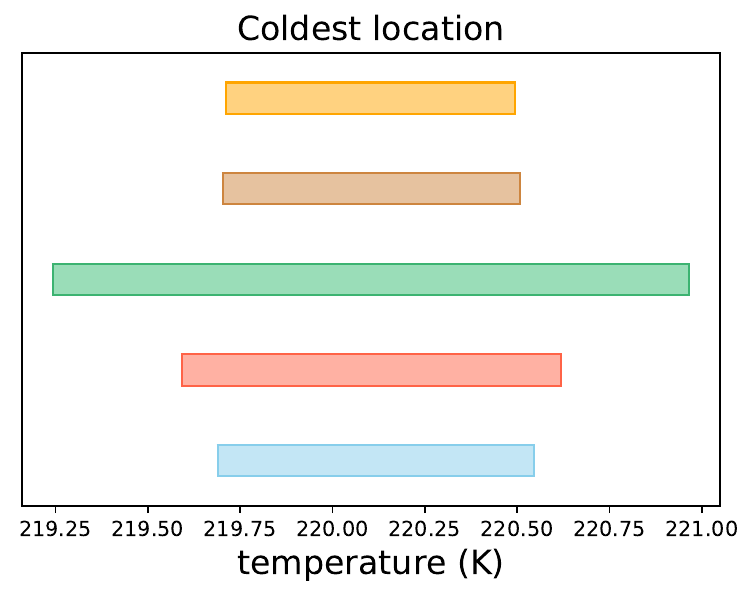}
\caption{Zoom correction intervals compared to the intervals computed via conditional \cite{lee2016exact}, hybrid \cite{andrews2019inference}, fully simultaneous, and locally simultaneous \cite{zrnic2022locally} inference, in extreme climate event analysis.}
\label{fig:climate}
\end{figure}

\begin{figure}[t]
\includegraphics[height=0.33\textwidth]{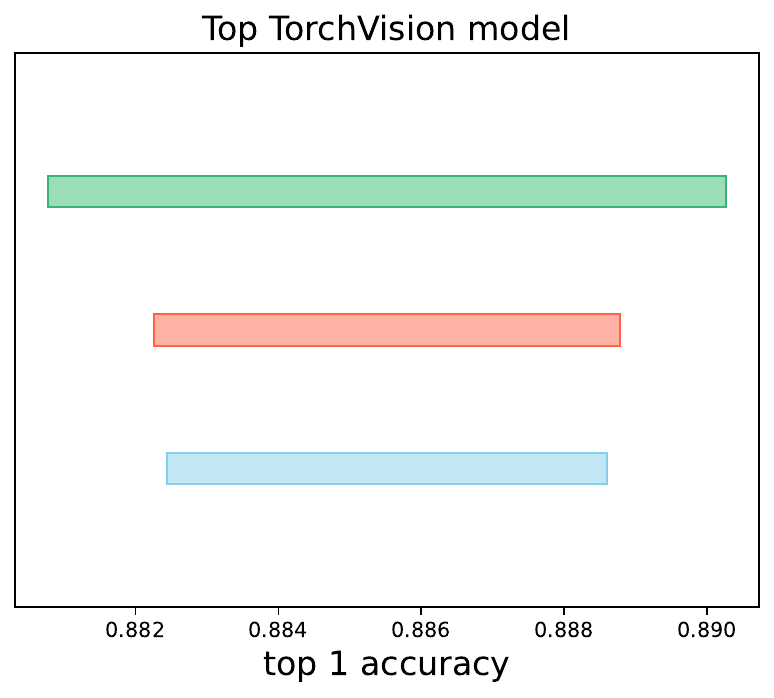}
\includegraphics[height=0.33\textwidth]{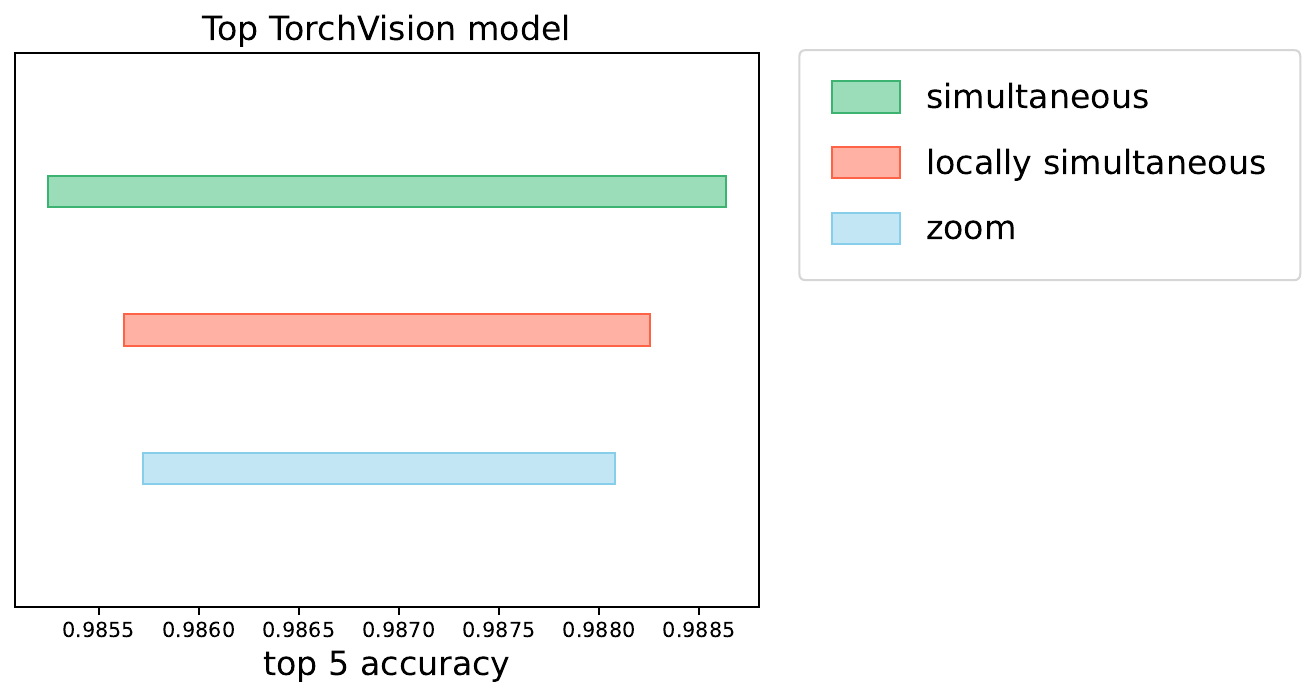}
\includegraphics[height=0.33\textwidth]{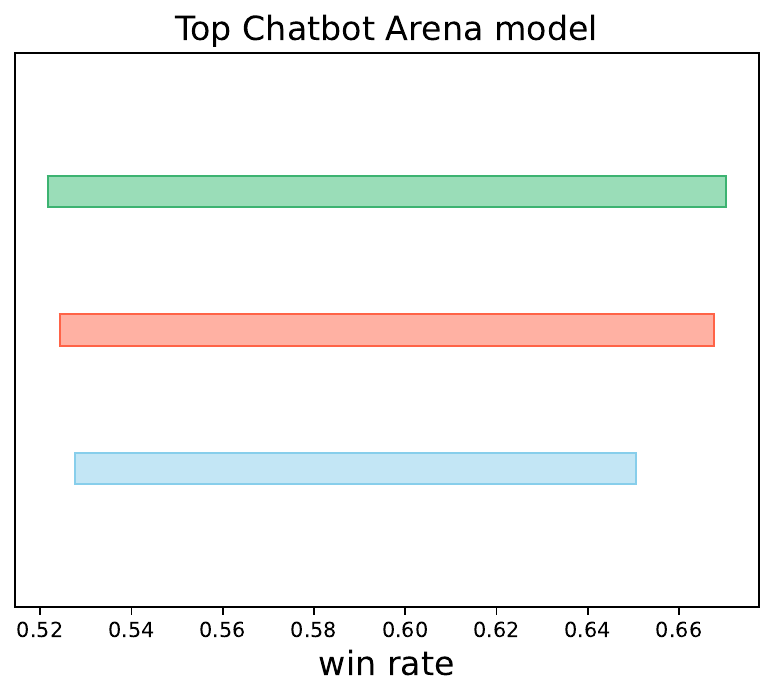}
\includegraphics[height=0.33\textwidth]{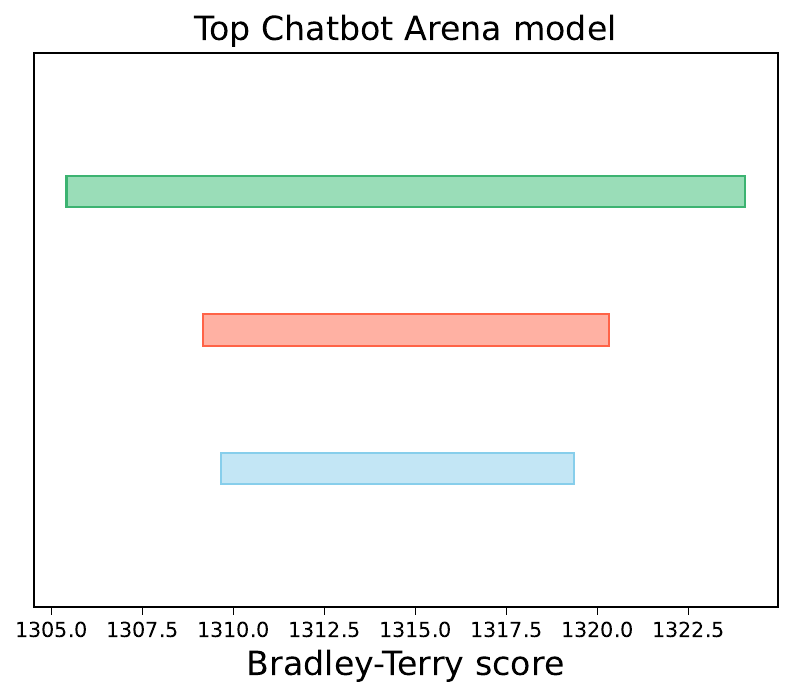}
\caption{Zoom correction intervals compared to the intervals computed via fully simultaneous and locally simultaneous \cite{zrnic2022locally} inference, in analyzing the performance of leaderboard winners.}
\label{fig:model_eval}
\end{figure}

\subsection{Model evaluation}

We consider inference on leaderboard winners; i.e., AI models that are the best performing on a benchmark. We consider two settings. In the first, we compare the accuracies of computer vision models available in TorchVision \cite{torchvision2016}, PyTorch's computer vision library, on the ImageNet dataset \cite{deng2009imagenet} (see \href{https://pytorch.org/vision/stable/models.html?fbclid=IwY2xjawFKrb9leHRuA2FlbQIxMAABHR_IjqeXFNGMex7cAqRt2Dusm9AguGW29-7C-oSYzBdLuTnDGtQ0Zy5SYQ_aem_qORwdM1YKothjcCN51LEqA}{here} for the accuracies). Our goal is to infer the population accuracy of the best performing model. We consider both top 1 and top 5 accuracy. In the second setting, we compare the performance of large language models (LLMs). We use data released by Chatbot Arena \cite{chiangchatbot}, a platform for crowd-sourced ranking of LLMs though pair-wise comparisons. Given a user-chosen prompt, the platform provides the answers of two LLMs (with hidden identities). The user then decides which answer they prefer. We compare the LLMs in terms of win rates, that is, how frequently they beat the competitors, as well as their Bradley-Terry scores \cite{bradley1952rank}, commonly used to quantify the ``strength'' of an LLM. Our goal is to estimate the average win rate and population value of the Bradley-Terry score of the best performing model (in both cases, ChatGPT-4o). Bradley-Terry scores are essentially logistic regression coefficients~\cite{chiangchatbot}; we thus apply a standard normal approximation for logistic regression to obtain an error tail bound.

Additionally, we apply the extension from Section \ref{sec:population_winner} to form a confidence set for the \emph{identity} of the population winner. That is, we return a set of models such that the population-best model in terms of the relevant metric (e.g., accuracy or Bradley-Terry score) is within the returned set.

Note that in both settings the scores observed for different models are highly dependent, and this dependence is difficult to characterize. In the first setting, the observed accuracies are dependent because all models are evaluated on the same test dataset from ImageNet. In the second setting, the win rates and Bradley-Terry scores of different models are dependent because they are collected through pairwise comparisons, so one ``battle'' affects the empirical win rates for two models. Given such dependence,~it is challenging to apply conditional and hybrid inference, and SoS does not apply since it assumes independent estimates. Therefore, we only compare the zoom correction to fully and locally simultaneous inference. For all three methods we assume worst-case dependence, i.e. we apply a union bound.

Figure \ref{fig:model_eval} shows the intervals. The zoom correction yields smaller intervals than both baselines, achieving a major improvement over the full union bound in all four examples. Locally simultaneous inference leads to noticeable improvements over the union bound in the TorchVision evaluations and Bradley-Terry score estimation, but does not see much gain for win rate estimation. In Table~\ref{table:torchvision} and Table~\ref{table:chatbot_arena} we report confidence sets for the identity of the population winner, obtained via the zoom correction and the full union bound. We see that the zoom correction reduces the size of the set in all problems but the top 1 accuracy in TorchVision and, in particular, it is able to identify the single best model on Chatbot Arena in terms of the Bradley-Terry score.

\begin{table*}[t!]
\renewcommand{\arraystretch}{1.5} 
\small
\centering
\begin{tabular}{l|cc}
\toprule
\multirow[t]{2}{*}{\textbf{Metric}} & \multicolumn{2}{c}{\textbf{Confidence set for the population winner (TorchVision)}} \\
                                    & zoom correction & simultaneous \\
\midrule
\textbf{Top 1 accuracy} & \parbox[t]{5.7cm}{\scriptsize\texttt{\{ViT\_H\_14\_Weights.IMAGENET1K\_SWAG\_E2E\_V1},\\ \texttt{RegNet\_Y\_128GF\_Weights.IMAGENET1K\_SWAG\_E2E\_V1},\\ \texttt{ViT\_L\_16\_Weights.IMAGENET1K\_SWAG\_E2E\_V1\}}} & \parbox[t]{5.7cm}{\scriptsize\texttt{\{ViT\_H\_14\_Weights.IMAGENET1K\_SWAG\_E2E\_V1},\\ \texttt{RegNet\_Y\_128GF\_Weights.IMAGENET1K\_SWAG\_E2E\_V1},\\ \texttt{ViT\_L\_16\_Weights.IMAGENET1K\_SWAG\_E2E\_V1\}}} \\
\\[-10pt] 
\textbf{Top 5 accuracy} & \parbox[t]{5.7cm}{\scriptsize\texttt{\{ViT\_H\_14\_Weights.IMAGENET1K\_SWAG\_E2E\_V1},\\ \texttt{RegNet\_Y\_128GF\_Weights.IMAGENET1K\_SWAG\_E2E\_V1},\\ \texttt{ViT\_L\_16\_Weights.IMAGENET1K\_SWAG\_E2E\_V1\}}} & \parbox[t]{5.7cm}{\scriptsize\texttt{\{ViT\_H\_14\_Weights.IMAGENET1K\_SWAG\_E2E\_V1},\\ \texttt{RegNet\_Y\_128GF\_Weights.IMAGENET1K\_SWAG\_E2E\_V1},\\ \texttt{ViT\_L\_16\_Weights.IMAGENET1K\_SWAG\_E2E\_V1},\\
\texttt{RegNet\_Y\_32GF\_Weights.IMAGENET1K\_SWAG\_E2E\_V1\}}} \\
\bottomrule
\end{tabular}
\caption{Confidence set for the population winner in terms of the top 1 and top 5 accuracy, among computer vision models in TorchVision.}
\label{table:torchvision}
\end{table*}

\begin{table*}[t!]
\renewcommand{\arraystretch}{1.5} 
\small
\centering
\begin{tabular}{l|cc}
\toprule
\multirow[t]{2}{*}{\textbf{Metric}} & \multicolumn{2}{c}{\textbf{Confidence set for the population winner (Chatbot Arena)}} \\
                                    & zoom correction & simultaneous \\
\midrule
\textbf{Win rate}                   & \parbox[t]{5cm}{\texttt{\{gpt-3.5-turbo-0314},\\ \texttt{gemini-1.5-pro-api-0409-preview},\\ \texttt{gemini-1.5-pro-exp-0801},\\ \texttt{gemini-1.5-pro-exp-0827},\\ \texttt{chatgpt-4o-latest},\\ \texttt{grok-2-2024-08-13\}}} & \parbox[t]{5cm}{\texttt{\{gpt-3.5-turbo-0314},\\ \texttt{gemini-1.5-pro-api-0409-preview},\\ \texttt{gemini-advanced-0514},\\ \texttt{gemini-1.5-pro-exp-0801},\\ \texttt{gemini-1.5-pro-exp-0827},\\ \texttt{chatgpt-4o-latest},\\ \texttt{gpt-4o-2024-08-06},\\ \texttt{grok-2-2024-08-13},\\ \texttt{llama-3.1-405b-instruct\}}} \\
\\[-10pt] 
\textbf{Bradley-Terry score}  & \parbox[t]{5cm}{\texttt{\{chatgpt-4o-latest\}}} & \parbox[t]{5cm}{\texttt{\{chatgpt-4o-latest},\\ \texttt{gemini-1.5-pro-api-0409-preview\}}} \\
\bottomrule
\end{tabular}
\caption{Confidence set for the population winner in terms of the win rate and Bradley-Terry score, among large language models on Chatbot Arena.}
\label{table:chatbot_arena}
\end{table*}

\subsection{Analysis of language devices}

In natural language processing, one is often interested in understanding the relationship between linguistic devices and the perceived sentiment of text; for example, the relationship between using devices such as gratitude words or hedging on the perceived politeness. A natural object of interest is devices that lead to the largest effect on the perceived politeness \cite{danescu2013computational}. We study the dataset in \cite{danescu2013computational}, which consists of online requests posted
on Stack Exchange and Wikipedia, annotated for politeness by human raters. All requests are additionally annotated with 21 binary features that correspond to the use of linguistic devices such as deference, gratitude, hedging, and so on.
We compute confidence intervals for the effect of the most effective devices. In the first setting we estimate the Pearson correlation coefficient between the use of each device and politeness and select the device that leads to the largest correlation; the winning device is using gratitude words. In the second setting, we regress politeness on all features via logistic regression and select the device with the largest coefficient; the winning device is indirectness.

In both cases, we apply a normal approximation. In the first problem, we do so via the Fisher z-transformation; in the second, we rely on the standard asymptotic normality of logistic regression. The dependencies between the different coefficients are intractable in the first problem and thus we apply simultaneous, locally simultaneous inference, and the zoom correction with a union bound. In the second problem, we can use the standard consistent plug-in estimate of the logistic regression covariance, allowing us to additionally apply the conditional and hybrid methods.

Figure \ref{fig:politeness} shows the computed intervals. In the first problem, the zoom correction yields a significantly smaller interval than both fully and locally simultaneous inference, which lead to similar intervals. In the second example, the winning coefficient stands out and thus conditional inference and the zoom correction yield essentially uncorrected intervals. The hybrid method and locally simultaneous inference give slightly larger intervals due to their use of error budget splitting, which prevents them from recovering exact uncorrected inference when the winner stands out.

\begin{figure}[t]
\hspace{2.2mm}\includegraphics[height=0.3\textwidth]{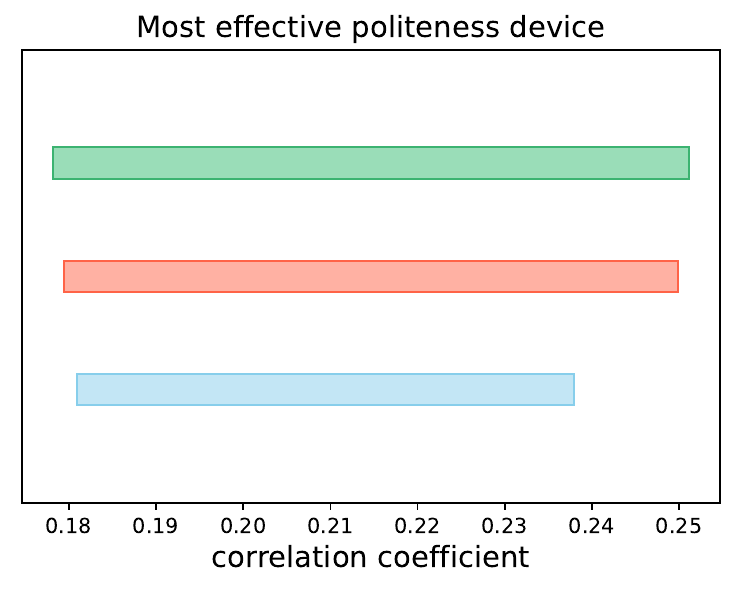}
\includegraphics[height=0.3\textwidth]{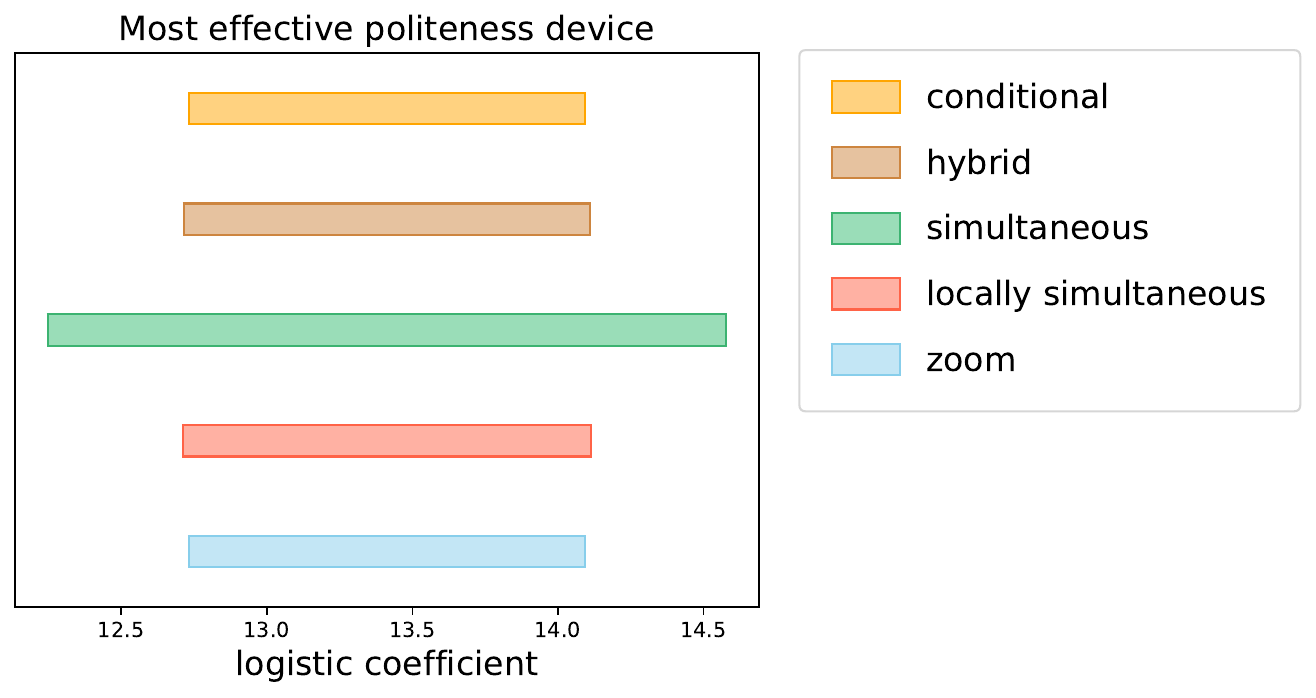}
\caption{Intervals computed via conditional \cite{lee2016exact}, hybrid \cite{andrews2019inference}, simultaneous, and locally simultaneous \cite{zrnic2022locally} inference, and the zoom correction, in language device analysis.}
\label{fig:politeness}
\end{figure}

\section{Discussion}
\label{sec:discussion}

We introduced the zoom correction: a novel, flexible approach for valid inference on the winner. The basic idea is to provide simultaneous inference for \emph{all} candidates in such a way that focuses the statistical power around the most promising candidates. The correction is applicable in both parametric and nonparametric settings, under possibly unknown dependence between the candidates.

In nonparametric settings or settings with unknown dependence, the zoom correction provides a strict improvement over the union bound. Moreover, it outperforms locally simultaneous inference empirically. We are not aware of other corrections that are applicable in such general settings.

In problems amenable to a normal approximation with a known covariance matrix, the hybrid method is typically the most powerful solution (though the zoom correction is often competitive). The main challenge for the hybrid method arises when there are many clearly suboptimal candidates, because the initial step in the hybrid approach---namely, ``classical'' simultaneous inference over all candidates---is unnecessarily loose around the promising candidates. However, notice that the hybrid method can be combined with \emph{any} simultaneous inference method---including the zoom correction. Using the zoom correction would give rise to a better hybrid method. Indeed, since the zoom correction strictly improves over classical simultaneous inference, a new hybrid method with the zoom correction as its initial step would outperform the usual hybrid method. In settings where a normal approximation with a known covariance matrix is applicable, we believe that this novel hybrid method building on the zoom correction would achieve the best of both worlds.

\section*{Acknowledgements}

We thank Isaiah Andrews for valuable feedback on this work and for proposing the integration of the zoom correction and the hybrid method, outlined in Section \ref{sec:discussion}.

\newpage

\bibliographystyle{plainnat}
\bibliography{refs}

\newpage

\appendix

\section{Proofs from Section \ref{sec:test}}

\subsection{Proof of Proposition \ref{prop:test-stat}}

\teststat*

\begin{proof}
We have
\begin{align*}
r_\alpha 
&= \min \left\{r:\; P\left(|\xi_j| \leq \max\left\{r,\frac{\Delta_j}{2}\right\}, \text{ for all } j\right) \geq 1-\alpha\right\} \\
&= \min \left\{r:\; P\left(|\xi_j|\,\cdot\, 1\left\{|\xi_j|> \frac{\Delta_j}{2}\right\} \leq r, \text{ for all } j\right) \geq 1-\alpha\right\} = \min \left\{r:\; P(M(\xi, \Delta) \leq r) \geq 1-\alpha\right\}.
\end{align*}
\end{proof}

\section{Proofs from Section \ref{sec:conf_int_winner}}
\label{sec:lemmas}

\subsection{Technical lemmas}

Below we state several technical lemmas that are used in the proof of Lemma \ref{lemma:mainlemma}.

First, we prove a convenient fact that allows us to restrict attention to $\theta$ vectors for which the winning coordinate $\ihat$ is an active coordinate.

\begin{lemma}[Winner must be active]
\label{lemma:winner_active}
Suppose $X \in A_\alpha(\theta)$, and let $\ihat = \argmax_{j\in[m]} \X_j$. Then, $\ihat \in \I_\alpha(\theta)$.
\end{lemma}

\begin{proof}
Let $i^*$ be a maximal index of $\theta$: $i^* \in \argmax_{j\in[m]}\theta_j$, and note that $i^*$ must be active because $\Delta_{i^*}=0$. Now suppose that, in fact, $\ihat\in \I_\alpha^c$. Then, we know
\[\X_{\ihat} \leq \theta_{\ihat} + \max\{r_\alpha, \Delta_{\ihat}/2\} = \theta_{\ihat} + \Delta_{\ihat}/2.\]
At the same time, because $i^*$ is active, we know
\[\X_{i^*} \geq \theta_{i^*} - r_\alpha.\]
Putting the two inequalities together, we get
\[\X_{\ihat} - \X_{i^*} \leq \theta_{\ihat} + \Delta_{\ihat}/2 - \theta_{i^*} + r_\alpha = -\Delta_{\ihat}/2 + r_\alpha.\]
But $\X_{\ihat} - \X_{i^*}\geq 0$ by the definition of $\ihat$, implying $r_\alpha \geq \Delta_{\ihat}/2$ and thus contradicting the claim $\ihat\in\I_\alpha^c$.
\end{proof}

\begin{lemma}[Only inactive gaps matter]
\label{lemma:r_only_inactive}
Fix $\theta \in \R^m$.
\begin{itemize}
\item[(i)] If $\theta'$ is any vector such that $\Delta_j(\theta') = \Delta_j(\theta)$ for $j\in \I_\alpha^c(\theta)$ and $\Delta_j(\theta') \leq \Delta_j(\theta)$ for $j\in \I_\alpha(\theta)$, then $r_\alpha(\theta) = r_\alpha(\theta')$.
\item[(ii)] If $\theta'$ is any vector such that $\Delta_j(\theta') \geq \Delta_j(\theta)$ for $j\in \I_\alpha^c(\theta)$, then $r_\alpha(\theta') \leq r_\alpha(\theta)$.
\end{itemize}
\end{lemma}

\begin{proof}
Recall that $r_\alpha(\theta) = \min\{r: \mathbf{S}(\max\{r,\Delta(\theta)/2\}) \leq \alpha\}$.

\emph{Proof of (i).} Since $\Delta_j(\theta') \leq \Delta_j(\theta)$ for all $j$,  $\max\{r,\Delta_j(\theta')/2\} \leq \max\{r,\Delta_j(\theta)/2\}$ for all $r$, and so $r_\alpha(\theta')$ has to be at least as large as $r_\alpha(\theta)$. However, since for $j\in \I_\alpha(\theta)$ we know $\max\{r_\alpha(\theta),\Delta_j(\theta)/2\} = r_\alpha(\theta)$, then we also have $\max\{r_\alpha(\theta),\Delta_j(\theta')/2\} = r_\alpha(\theta)$ and $\mathbf{S}(\max\{r_\alpha(\theta),\Delta(\theta')/2\}) \leq \alpha$. Consequently, $r_\alpha(\theta)$ satisfies the definition of the active radius for $\theta'$, and so $r_\alpha(\theta') = r_\alpha(\theta)$.

\emph{Proof of (ii).} Let $\theta''$ be any vector such that:
\[\Delta_j(\theta'')=\begin{cases} \Delta_j(\theta) & \text{if } j\in \I^c_\alpha(\theta);\\
0 & \text{if }  j \in \I_\alpha(\theta).
\end{cases}\]
Then, by claim (i), $r_\alpha(\theta) = r_\alpha(\theta'')$. However, since $\theta''$ has gaps that are at least as small as those of $\theta'$ coordinate-wise, we know $r_\alpha(\theta'') \geq r_\alpha(\theta')$. Therefore, $r_\alpha(\theta) \geq r_\alpha(\theta')$.
\end{proof}

\begin{lemma}[Increasing inactive coordinates cannot increase all their radii]
\label{lemma:increasing_coord}
Fix a vector $\theta\in\R^m$ and an index set $\I\subseteq \I_\alpha^c(\theta)$. Let $\theta'$ be any vector such that $\Delta_j(\theta) \leq \Delta_j(\theta')$ for all $j\in\I_\alpha^c(\theta) \setminus \I$ and  $\Delta_j(\theta') < \Delta_j(\theta)$ for $j\in\I$. Let 
$i_0 = \argmax_{j\in\I} \Delta_j(\theta).$
Then, $\Delta_{i_0}(\theta)/2 \geq \max\{r_\alpha(\theta'), \Delta_{i_0}(\theta')/2\}$.
\end{lemma}

\begin{proof}
Suppose the claim is not true. Then, since $\Delta_{i_0}(\theta') < \Delta_{i_0}(\theta)$, that must mean that
\[r_\alpha(\theta') > \Delta_{i_0}(\theta)/2 = \max_{j\in\I} \Delta_j(\theta)/2 > r_\alpha(\theta),\]
where the last inequality follows from the fact that $i_0$ is inactive at $\theta$. This means that 
\[\max\{\Delta_j(\theta')/2, r_\alpha(\theta')\} > \max\{\Delta_j(\theta)/2, r_\alpha(\theta)\} \]
for all $j\in\I \cup \I_\alpha(\theta)$. Moreover, for the remaining coordinates, $j\in\I^c_\alpha(\theta)\setminus \I$, we know that 
\[\max\{\Delta_j(\theta')/2, r_\alpha(\theta')\} \geq \Delta_j(\theta')/2 \geq \Delta_j(\theta)/2 = \max\{\Delta_j(\theta)/2, r_\alpha(\theta)\}.\]
In conclusion, the radius of the acceptance region $A_\alpha(\theta')$ is at least as large as the radius of the acceptance region $A_\alpha(\theta)$ in all coordinates $j$, and for some coordinates it is strictly larger. That means $P(X\in A_\alpha(\theta')) > P(X\in A_\alpha(\theta)) \geq 1-\alpha$, which in turn means that we could decrease the radius $r_\alpha(\theta')$ while maintaining the coverage at level $1-\alpha$. This contradicts the assumption that $r_\alpha(\theta')$ is minimal, as per Eq.~\eqref{eq:active-radius}; therefore, the claim of the lemma must be correct.
\end{proof}

\subsection{Proof of Lemma \ref{lemma:mainlemma}}

\mainlemma*

\begin{proof}
$(\Leftarrow)$ By the definition of the projection interval $\widehat C_{\ihat}^\alpha$, it suffices to show that $\X \in A_\alpha(\theta^t)$, since $\theta_{\ihat}^t = t$. In other words, we show that for all $j\in[m]$ we have $|\X_j-\theta_j^t| \leq \max\{r_\alpha(\theta^t), \Delta_j(\theta^t)/2\}$. This is true for coordinate $\ihat$ by assumption:
\[
|\X_{\ihat} - \theta_{\ihat}^t| = |\X_{\ihat} - t| \leq r_\alpha(\theta^t).
\]
For all other coordinates $j\neq \ihat$, we split the analysis depending on whether $\X_j \leq t$ or not. If $\X_j \leq t$, we have
\begin{equation}
\label{eq:thetat_eq1}
|\X_j - \theta_j^t| = \frac{t - \X_j}{3} = \frac{t - \theta_j^t}{2} = \frac{\Delta_j(\theta^t)}{2}.
\end{equation}
If $\X_j > t$, then we have 
\begin{equation}
\label{eq:thetat_eq2}
|\X_j -  \theta_j^t| = \X_j - t \leq \X_{\hat i} - t \leq r_\alpha(\theta^t),
\end{equation}
where the last inequality is true by assumption. Therefore, we have shown $|\X_j-\theta_j^t| \leq \max\{r_\alpha(\theta^t), \Delta_j( \theta^t)/2\}$ for all $j\in[m]$, and thus $X\in A_\alpha( \theta^t)$.\\

\noindent $(\Rightarrow)$ The fact that $t\in \widehat C_{\ihat}^\alpha$ means that  there exists some $\theta\in\R^m$ with $\theta_{\ihat}=t$ for which $\X \in A_\alpha(\theta)$. We use Lemma \ref{lemma:winner_active}, which says that $\ihat$ must be an active index for any $\theta$ such that $\X\in A_\alpha(\theta)$. Therefore, we know $r_\alpha(\theta) \geq \Delta_{\ihat}(\theta)/2$, implying
\begin{equation}
\label{eq:rtheta}
|\X_{\ihat} - \theta_{\ihat}| = |\X_{\ihat} - t| \leq \max\{r_\alpha(\theta), \Delta_{\ihat}(\theta)/2\} = r_\alpha(\theta).
\end{equation}
If we show $r_\alpha(\theta^t) \geq r_\alpha(\theta)$ the proof is complete, so now we focus on showing this fact.

Lemma \ref{lemma:r_only_inactive} (ii) gives a key characterization of the active radius: $r_\alpha(\theta)$ is non-increasing in the \emph{inactive} gaps. Therefore, if we prove that $\theta^t$ has the smallest inactive gaps of all $\theta$ such that $X\in A_\alpha(\theta)$, then the result follows. Formally, if $\Delta_j(\theta) \geq \Delta_j( \theta^t)$ holds for all $j\in \I^c_\alpha( \theta^t)$, then $r_\alpha( \theta^t) \geq r_\alpha(\theta)$. 

Suppose this is not true: for some $j\in \I^c_\alpha( \theta^t)$, $\Delta_j(\theta) < \Delta_j( \theta^t)$. First note that this means that $\theta_j >  \theta_j^t$, because the maximum of $\theta$ is at least as large as the maximum of $ \theta^t$: $\max_j \theta_j \geq \theta_{\ihat} = \max_j  \theta_j^t = t$. Note also that it must be that $X_j\leq t$ and thus $ \theta_j^t$ must be equal to $\frac{2}{3} \X_j + \frac 1 3 t$; if not,  $\X_j > t$ would imply $\Delta_j( \theta^t)=  \theta^t_{\ihat} -  \theta_j^t = t-t = 0$, contradicting the assumption that $j$ is inactive, $j\in \I^c_\alpha( \theta^t)$. Therefore, 
\begin{equation*}
\label{eq:inactive_gap}
\theta_j - \X_j >  \theta_j^t - \X_j = \frac 1 3 (t-\X_j)= \frac 1 2 (t- \theta_j^t) = \frac{\Delta_j( \theta^t)}{2}.
\end{equation*}
Lemma \ref{lemma:increasing_coord} shows that there must be \emph{at least one} $j$ for which $\Delta_j(\theta) < \Delta_j(\theta^t)$, such that $\Delta_j( \theta^t)/2 \geq  \max\{r_\alpha(\theta), \Delta_j(\theta)/2\}$. As a result, there must exist a coordinate $j$ such that $\theta_j - \X_j > \max\{r_\alpha(\theta), \Delta_j(\theta)/2\}$, which contradicts the claim that $X\in A_\alpha(\theta)$. Therefore, we must in fact have $\Delta_j(\theta) \geq \Delta_j( \theta^t)$ for all $j\in \I^c_\alpha( \theta^t)$.

Finally, $r_\alpha( \theta^t) \geq r_\alpha(\theta)$ follows by invoking Lemma \ref{lemma:r_only_inactive} (ii). Going back to Eq. \eqref{eq:rtheta} and applying this inequality completes the proof.
\end{proof}

\section{Proofs from Section \ref{sec:stepdown}}

\subsection{Proof of Lemma \ref{lemma:endpoint_equation}}

\endpointequation*

\begin{proof}
We state the proof for the lower endpoint; the proof for the upper endpoint proceeds analogously.

Recall the definition of the worst-case vector $\theta^t$ that exactly characterizes which $t$ are in $\widehat C^\alpha$:
\begin{equation*}
\theta_j^t = \begin{cases} t & \text{if } j = \ihat; \\ \min\left\{\frac{2}{3}\X_j + \frac{1}{3}t, t \right\} & \text{if } j \neq \ihat.\end{cases}
\end{equation*}
We know that by definition $r_l \equiv r_\alpha(\theta^{t_l})$ satisfies
\[\alpha = \sum_{j=1}^m S\left(\max\left\{r_\alpha(\theta^{t_l}), \frac{\Delta_j(\theta^{t_l})}{2}\right\}\right).\]
The gaps are equal to $\Delta_j(\theta^{t_l}) = \max\{0, \frac{2}{3}(t_l-X_j)\} = \max\{0, \frac{2}{3}(t_l - X_{\ihat} + \hDelta_j)\}$. Since by the continuity of the noise distribution we know that $t_l$ must satisfy the boundary condition $t_l = X_{\ihat} - r_\alpha(\theta^{t_l})$, this in turn implies the claim of the lemma:
\[\alpha = \sum_{j=1}^m S\left(\max\left\{r_\alpha(\theta^{t_l}), \frac{\hDelta_j - r_\alpha(\theta^{t_l})}{3}\right\}\right).\]
\end{proof}

\subsection{Proof of Theorem \ref{thm:stepdown}}

\stepdown*

\begin{proof}
To show the validity of the lower bound, we rely on the fact that $\S_l(r_l)=\alpha$. We show $\hat r_l \geq r_l$ by proving that $\S_l(\hat r_l) \leq \alpha$ \emph{and} $\S_l(r)<\alpha$ for all $r> \hat r_l$. The proof for the upper bound is simpler, since $\S_u(r_u)=\alpha$ \emph{and} $\S_u(r)$ is monotonically decreasing; we simply show $\S_u(\hat r_u)\leq \alpha$ and that implies $\hat r_u \geq r_u$.

\paragraph{Lower bound.}
Let $k_0\in[m]$ be the first step when $\hDelta_{(j)} \leq 4\hat r_l^{(j)}$ and the algorithm terminates (such $k_0$ must exist because $\hDelta_{(m)}=0$). We split the terms before $k_0$ and starting with $k_0$:
\begin{align*}
\S_l(\hat r_l) &= \sum_{j=1}^{k_0-1} S\left(\max\left\{\hat r_l, \frac{\hDelta_{(j)} - \hat r_l}{3}\right\}\right) + \sum_{j=k_0}^{m} S\left(\max\left\{\hat r_l, \frac{\hDelta_{(j)} - \hat r_l}{3}\right\}\right).
\end{align*}
At every step $k \leq k_0$, Algorithm \ref{alg:stepdownlow} ensures
\begin{align}
\label{eq:alphak}
\sum_{j=k}^{m} S\left(\hat r_l^{(k)}\right) = \alpha_{k} = \alpha - \sum_{j=1}^{k-1} S\left(\frac{\hDelta_{(j)} - \hat r_l^{(j)}}{3}\right).
\end{align}
Since $\hat r_l \equiv \hat r_l^{(k_0)}$, and furthermore $\hat r_l \geq  \frac{\hDelta_{(j)} - \hat r_l}{3}$ for all $j\geq k_0$, we thus have
\begin{align*}\S_l(\hat r_l) &= \alpha + \sum_{j=1}^{k_0-1} \left(S\left(\max\left\{\hat r_l, \frac{\hDelta_{(j)} - \hat r_l}{3}\right\}\right)  - S\left(\frac{\hDelta_{(j)} - \hat r_l^{(j)}}{3}\right) \right)\\
&\leq \alpha + \sum_{j=1}^{k_0-1} \left(S\left( \frac{\hDelta_{(j)} - \hat r_l}{3}\right)  - S\left(\frac{\hDelta_{(j)} - \hat r_l^{(j)}}{3}\right) \right).
\end{align*}
Therefore, if we show that $\hat r_l \equiv \hat r_l^{(k_0)} \leq \hat r_l^{(j)}$ for $j< k_0$, then $S\left( \frac{\hDelta_{(j)} - \hat r_l}{3}\right)  \leq  S\left(\frac{\hDelta_{(j)} - \hat r_l^{(j)}}{3}\right)$ for each term $j$ and we have $\S_l(\hat r_l) \leq \alpha$, as desired. Since $\frac{\hDelta_{(j)} - \hat r_l^{(j)}}{3} \geq \hat r_l^{(j)}$ for $j<k_0$, we have $S\left(\frac{\hDelta_{(j)} - \hat r_l^{(j)}}{3}\right) \leq S(\hat r_l^{(j)}) \leq \frac{ \alpha_j}{m-j+1}$, by the definition of $\hat r_l^{(j)}$. This in turn implies
\[\frac{\alpha_{j+1}}{m-j} = \frac{1}{m-j}\left( \alpha_j - S\left(\frac{\hDelta_{(j)} - \hat r_l^{(j)}}{3}\right)\right) \geq \frac{1}{m-j} \alpha_j \left(1 - \frac{1}{m-j+1}\right) = \frac{\alpha_j}{m-j+1}, \]
and thus $\hat r_l^{(j+1)} = S^{-1}(\frac{\alpha_{j+1}}{m-j}) \leq S^{-1}(\frac{\alpha_{j}}{m-j+1}) = \hat r_l^{(j)}$. This proves that $\hat r_l^{(j)}$ decreases as $j$ increases, implying $\hat r_l^{(k_0)} \leq \hat r_l^{(j)}$, as desired.

Finally, we argue that $\S_l(r) < \alpha$ for any $r>\hat r_l$. For $r > \hat r_l^{(1)}$, the claim is immediate, since
$$\S_l(r) \leq m S(r) < m S(\hat r_l^{(1)}) = \alpha.$$
Otherwise, assume $r\in(\hat r_l^{(k+1)}, \hat r_l^{(k)}]$ for some $k\in\{1,\dots,k_0-1\}$. Then, we have
\begin{align*}
\S_l(r) \leq \sum_{j=1}^{k} S\left(\frac{\hDelta_{(j)}-r}{3}\right) + \sum_{j=k+1}^m S(r) &< \sum_{j=1}^{k} S\left( \frac{\hDelta_{(j)}-r_l^{(k)}}{3}\right) + \sum_{j=k+1}^{m} S(r_l^{(k+1)}) \\
&\leq \sum_{j=1}^{k} S\left( \frac{\hDelta_{(j)}-r_l^{(j)}}{3}\right) + \sum_{j=k+1}^{m} S(r_l^{(k+1)}).
\end{align*}
By Eq.~\eqref{eq:alphak}, the right-hand side is at most $\alpha$ for all $k< k_0$, and thus $\S_l(r)<\alpha$ for all $r>\hat r_l$.

\paragraph{Upper bound.}
We proceed with a similar argument as for the lower bound.
Let $k_0\in[m]$ be the first step when $\hDelta_{(j)} \leq 2\hat r_u^{(j)}$ and the algorithm terminates (such $k_0$ must exist because $\hat \Delta_{(m)}=0$). We split the terms before $k_0$ and starting with $k_0$:
\[\S_u(\hat r_u) = \sum_{j=1}^{k_0-1} S\left(\max\left\{\hat r_u, \frac{\hDelta_{(j)} + \hat r_u}{3}\right\}\right) + \sum_{j=k_0}^{m} S\left(\max\left\{\hat r_u, \frac{\hDelta_{(j)} + \hat r_u}{3}\right\}\right).\]
Since $\hat r_u \equiv \hat r_u^{(k_0)}$, and furthermore $\hat r_u \geq  \frac{\hDelta_{(j)} + \hat r_u}{3}$ for all $j\geq k_0$, we can write the second sum as
\[\sum_{j=k_0}^{m} S\left(\hat r_u^{(k_0)}\right) = \alpha_{k_0} = \alpha - \sum_{j=1}^{k_0-1} S\left(\frac{\hDelta_{(j)} +  S^{-1}(\alpha)}{3}\right),\]
where the last step uses the definition of $\alpha_{k_0}$. Putting everything together, we have shown
\begin{align*}
\S_u(\hat r_u) &= \alpha + \sum_{j=1}^{k_0-1} \left(S\left(\max\left\{\hat r_u, \frac{\hDelta_{(j)} + \hat r_u}{3}\right\}\right)  - S\left(\frac{\hDelta_{(j)} +  S^{-1}(\alpha)}{3}\right) \right)\\
&\leq \alpha + \sum_{j=1}^{k_0-1} \left(S\left( \frac{\hDelta_{(j)} + \hat r_u}{3}\right)  - S\left(\frac{\hDelta_{(j)} + S^{-1}(\alpha)}{3}\right) \right).
\end{align*}
We have that $\hat r_u^{(j)} \geq S^{-1}(\alpha)$ since $\frac{\alpha_j}{m-j+1}\leq \alpha$ for all $j\in[m]$; therefore, $S\left( \frac{\hDelta_{(j)} + \hat r_u}{3}\right)  \leq S\left(\frac{\hDelta_{(j)} + S^{-1}(\alpha)}{3}\right)$. This completes the proof that $\S_u(\hat r_u)\leq \alpha$.
\end{proof}

\section{Proofs from Section \ref{sec:topk}}

\subsection{Technical lemmas}

We begin by stating several supporting technical lemmas, analogous to those in Section \ref{sec:lemmas} for the single winner case.

\begin{lemma}[Winners must be active]
\label{lemma:winners_active-topk}
Suppose $X \in A_\alpha$, and let $\ihat_{(1)},\dots,\ihat_{(k)}$ be the top $k$ coordinates of $X$. Then, $\ihat_{(1)},\dots,\ihat_{(k)} \in \I_\alpha$.
\end{lemma}

\begin{proof}
Let $i_{(1)}^*,\dots,i_{(k)}^*$ denote the top $k$ coordinates of $\theta$, and note that all of them must be active because $\Delta_{i_{(j)}^*}^{(k)}=0$ for all $j\in[k]$. Now suppose that, in fact, $\ihat_{(j)}\in \I_\alpha^c$ for some $j\in[k]$. Then, we know
\[\X_{\ihat_{(j)}} \leq \theta_{\ihat_{(j)}} + \max\{r_\alpha, \Delta^{(k)}_{\ihat_{(j)}}/2\} = \theta_{\ihat_{(j)}} + \Delta^{(k)}_{\ihat_{(j)}}/2.\]
At the same time, because $i_{(j')}^*$ are active for all $j'\in[k]$, we know
\[\X_{i_{(j')}^*} \geq \theta_{i_{(j')}^*} - r_\alpha.\]
Putting the two inequalities together, we get
\[\X_{\ihat_{(j)}} - \X_{i^*_{(j')}} \leq \theta_{\ihat_{(j)}} + \Delta_{\ihat_{(j)}}^{(k)}/2 - \theta_{i_{(j')}^*} + r_\alpha \leq -\Delta_{\ihat_{(j)}}^{(k)}/2 + r_\alpha\]
for all $j'\in[k]$. But $\X_{\ihat_{(j)}} - \X_{i^*_{(j')}}\geq 0$ for at least one $j'\in[k]$ by the definition of $\ihat_{(j)}$, implying $r_\alpha \geq \Delta_{\ihat_{(j)}}^{(k)}/2$ and thus contradicting the claim that $\ihat_{(j)}\in\I_\alpha^c$.
\end{proof}

\begin{lemma}[Only inactive gaps matter]
\label{lemma:r_only_inactive-topk}
Fix $\theta \in \R^m$.
\begin{itemize}
\item[(i)] If $\theta'$ is any vector such that $\Delta_j^{(k)}(\theta') = \Delta_j^{(k)}(\theta)$ for $j\in \I_\alpha^c(\theta)$ and $\Delta_j^{(k)}(\theta') \leq \Delta^{(k)}_j(\theta)$ for $j\in \I_\alpha(\theta)$, then $r_\alpha(\theta) = r_\alpha(\theta')$.
\item[(ii)] If $\theta'$ is any vector such that $\Delta_j^{(k)}(\theta') \geq \Delta_j^{(k)}(\theta)$ for $j\in \I_\alpha^c(\theta)$, then $r_\alpha(\theta') \leq r_\alpha(\theta)$.
\end{itemize}
\end{lemma}
\noindent The proof of Lemma \ref{lemma:r_only_inactive-topk} is identical to the proof of Lemma \ref{lemma:r_only_inactive} once we replace $\Delta_j$ with $\Delta_j^{(k)}$, and we therefore omit it.

\begin{lemma}[Increasing inactive coordinates cannot increase all their radii]
\label{lemma:increasing_coord-topk}
Fix a vector $\theta\in\R^m$ and an index set $\I\subseteq \I_\alpha^c(\theta)$. Let $\theta'$ be any vector such that $\Delta^{(k)}_j(\theta) \leq \Delta^{(k)}_j(\theta')$ for all $j\in\I_\alpha^c(\theta) \setminus \I$ and  $\Delta^{(k)}_j(\theta') < \Delta^{(k)}_j(\theta)$ for $j\in\I$. Let 
$i_0 = \argmax_{j\in\I} \Delta^{(k)}_j(\theta).$
Then, $\Delta^{(k)}_{i_0}(\theta)/2 \geq \max\{r_\alpha(\theta'), \Delta^{(k)}_{i_0}(\theta')/2\}$.
\end{lemma}
\noindent Once we replace $\Delta_j$ with $\Delta_j^{(k)}$, Lemma \ref{lemma:increasing_coord-topk} also follows the exact same proof as its single-winner counterpart, Lemma \ref{lemma:increasing_coord}. Hence we omit the proof.

\subsection{Proof of Lemma \ref{lemma:main-lemma-topk}}

\mainlemmatopk*

\begin{proof}
$(\Leftarrow)$ By the definition of the projection region $\widehat C_{\ihat_{(1)},\dots,\ihat_{(k)}}^\alpha$, it suffices to show that $\X \in A_\alpha(\theta^{\tvec})$, since $\theta^{\tvec}_{\ihat_{(j)}} = t_{\ihat_{(j)}}$ for all $j\in[k]$. In other words, we show that for all $j\in[m]$ we have $|\X_j-\theta_j^{\tvec}| \leq \max\{r_\alpha(\theta^{\tvec}), \Delta_j^{(k)}(\theta^{\tvec})/2\}$. This is true for coordinates $\{\ihat_{(1)},\dots,\ihat_{(k)}\}$ by assumption:
\[
|\X_{\ihat_{(j)}} - \theta_{\ihat_{(j)}}^{\tvec}| = |\X_{\ihat_{(j)}} - t_{\ihat_{(j)}}| \leq r_\alpha(\theta^{\tvec}),
\]
for all $j\in[k]$. For all other coordinates $j\not\in\{{\ihat_{(1)}},\dots,{\ihat_{(k)}}\}$, we split the analysis depending on whether $\X_j \leq\tvec_{(k)}$ or not. If $\X_j \leq\tvec_{(k)}$, we have
\[
|\X_j - \theta_j^{\tvec}| = \frac{t_{(k)} - \X_j}{3} = \frac{t_{(k)} - \theta_j^{\tvec}}{2} = \frac{\Delta_j^{(k)}(\theta^{\tvec})}{2}.
\]
If $\X_j >\tvec_{(k)}$, noting that we also know $X_j \leq \min_{j\in[k]} X_{\ihat_{(j)}}$ and letting $k_0 = \argmin_{j\in[k]} t_{\ihat_{(j)}}$, then we have 
$$|\X_j -  \theta_j^{\tvec}| = \X_j -\tvec_{(k)} \leq \X_{\ihat_{k_0}} -\tvec_{(k)} \leq r_\alpha(\theta^{\tvec}),$$
where the last inequality is true by assumption. Therefore, we have shown $|\X_j-\theta_j^{\tvec}| \leq \max\{r_\alpha(\theta^{\tvec}), \Delta_j^{(k)}( \theta^{\tvec})/2\}$ for all $j\in[m]$, and thus $X\in A_\alpha( \theta^{\tvec})$.\\

\noindent $(\Rightarrow)$ The fact that $\tvec \in \widehat C_{\ihat_{(1)},\dots,\ihat_{(k)}}^\alpha$ means that  there exists some $\theta\in\R^m$ with $\theta_{\ihat_{(j)}}=t_{\ihat_{(j)}}$ for all $j\in[k]$, for which $\X \in A_\alpha(\theta)$. We use Lemma \ref{lemma:winners_active-topk}, which says that $\ihat_{(1)},\dots,\ihat_{(k)}$ must be active indices for any $\theta$ such that $\X\in A_\alpha(\theta)$. Therefore, we know $r_\alpha(\theta) \geq \Delta^{(k)}_{\ihat_{(j)}}(\theta)/2$ for all $j\in[k]$, implying
\begin{equation}
\label{eq:rtheta2}
|\X_{\ihat_{(j)}} - \theta_{\ihat_{(j)}}| = |X_{\ihat_{(j)}} - t_{\ihat_{(j)}}| \leq \max\{r_\alpha(\theta), \Delta^{(k)}_{\ihat_{(j)}}(\theta)/2\} = r_\alpha(\theta).
\end{equation}
If we show $r_\alpha(\theta^{\tvec}) \geq r_\alpha(\theta)$ the proof is complete, so now we focus on showing this fact.

Lemma \ref{lemma:r_only_inactive-topk} (ii) gives a key characterization of the active radius: $r_\alpha(\theta)$ is non-increasing in the \emph{inactive} gaps. Therefore, if we prove that $\theta^{\tvec}$ has the smallest inactive gaps of all $\theta$ such that $X\in A_\alpha(\theta)$, then the result follows. Formally, if $\Delta_j^{(k)}(\theta) \geq \Delta_j^{(k)}( \theta^{\tvec})$ holds for all $j\in \I^c_\alpha( \theta^{\tvec})$, then $r_\alpha( \theta^{\tvec}) \geq r_\alpha(\theta)$. 

Suppose this is not true: for some $j\in \I^c_\alpha( \theta^{\tvec})$, $\Delta_j^{(k)}(\theta) < \Delta_j^{(k)}( \theta^{\tvec})$. First note that this means that $\theta_j >  \theta_j^{\tvec}$, because the $k$-th largest coordinate of $\theta$, $\theta_{(k)}$, is at least as large as the $k$-th largest coordinate of $ \theta^{\tvec}$: $\theta_{(k)} \geq \min_{j\in[k]} \theta_{\ihat_{(j)}} = \min_{j\in[k]} \theta_{\ihat_{(j)}}^{\tvec}  = \tvec_{(k)}$. Note also that it must be that $X_j\leq \tvec_{(k)}$ and thus $ \theta_j^{\tvec}$ must be equal to $\frac{2}{3} \X_j + \frac 1 3 \tvec_{(k)}$; if not,  $\X_j > \tvec_{(k)}$ would imply $\Delta^{(k)}_j( \theta^{\tvec})=  \max\{\tvec_{(k)} -  \theta_j^{\tvec}, 0\} = 0$, contradicting the assumption that $j$ is inactive, $j\in \I^c_\alpha( \theta^{\tvec})$. Therefore, 
\begin{equation*}
\label{eq:inactive_gap-topk}
\theta_j - \X_j >  \theta_j^{\tvec} - \X_j = \frac 1 3 (\tvec_{(k)}-\X_j)= \frac 1 2 (\tvec_{(k)}- \theta_j^{\tvec}) = \frac{\Delta_j^{(k)}( \theta^{\tvec})}{2}.
\end{equation*}
Lemma \ref{lemma:increasing_coord-topk} shows that there must be \emph{at least one} $j$ for which $\Delta_j^{(k)}(\theta) < \Delta_j^{(k)}(\theta^{\tvec})$, such that $\Delta_j^{(k)}( \theta^{\tvec})/2 \geq  \max\{r_\alpha(\theta), \Delta_j^{(k)}(\theta)/2\}$. As a result, there must exist a coordinate $j$ such that $\theta_j - \X_j > \max\{r_\alpha(\theta), \Delta_j^{(k)}(\theta)/2\}$, which contradicts the claim that $X\in A_\alpha(\theta)$. Therefore, we must in fact have $\Delta_j^{(k)}(\theta) \geq \Delta_j^{(k)}( \theta^{\tvec})$ for all $j\in \I^c_\alpha( \theta^{\tvec})$.

Finally, $r_\alpha( \theta^{\tvec}) \geq r_\alpha(\theta)$ follows by invoking Lemma \ref{lemma:r_only_inactive-topk} (ii). Going back to Eq. \eqref{eq:rtheta2} and applying this inequality completes the proof.
\end{proof}

\subsection{Proof of Lemma \ref{lemma:topk-1d}}

\topkoned*

\begin{proof}
Lemma \ref{lemma:main-lemma-topk} implies $\max_{j\in[k]} |r_{\ihat_{(j)}}| \leq r_\alpha(\theta^{\tvec})$, where $\tvec = (t_{\ihat_{(1)}},\dots,t_{\ihat_{(k)}}) = (X_{\ihat_{(1)}}-r_{\ihat_{(1)}},\dots,X_{\ihat_{(k)}}-~r_{\ihat_{(k)}}).$
Let 
\[\tilde{\tvec} = (\tilde t_{\ihat_{(1)}},\dots,\tilde t_{\ihat_{(k)}}) = (X_{\ihat_{(1)}}-|r_{\ihat_{(j_0)}}|,\dots,X_{\ihat_{(k)}}-|r_{\ihat_{(j_0)}}|), \text{ where } j_0 = \argmax_{j\in[k]} |r_{\ihat_{(j)}}|.\]
Clearly, $r_\alpha(\theta^{\tilde{\tvec}}) \geq r_\alpha(\theta^{\tvec})$ since the gaps induced by $\tilde{\tvec}$ can only be smaller. Therefore, by Lemma \ref{lemma:main-lemma-topk} we have
\[|r_{\ihat_{(j_0)}}| = \max_{j\in[k]}|X_{\ihat_{(j)}} - \tilde t_{\ihat_{(j)}} | = \max_{j\in[k]}|X_{\ihat_{(j)}} -  t_{\ihat_{(j)}} | \leq r_\alpha(\theta^{\tilde{\tvec}}) = r_\alpha(\tilde \theta^{|r_{\ihat_{(j_0)}}|}),\]
where the last step simply uses the fact that $\theta^{\tilde\tvec} = \tilde \theta^{|r_{\ihat_{(j_0)}}|}$.
\end{proof}

\subsection{Proof of Lemma \ref{lemma:endpoint_equation_topk}}

\endpointeqtopk*

\begin{proof}
Recall the definition of $\tilde \theta^r$:
\begin{equation*}
\tilde\theta_j^r = \begin{cases} \X_j-r & \text{if } j \in \{\ihat_{(1)},\dots,\ihat_{(k)}\}; \\ \min\left\{\frac{2}{3}\X_j + \frac{1}{3}(\X_{\ihat_{(k)}} - r), \X_{\ihat_{(k)}} - r \right\} & \text{if } j \not\in \{\ihat_{(1)},\dots,\ihat_{(k)}\}.\end{cases}
\end{equation*}
For any $r$, we know that by definition $r_\alpha(\tilde \theta^r)$ satisfies
\[\alpha = \sum_{j=1}^m S\left(\max\left\{r_\alpha(\tilde \theta^r), \frac{\Delta^{(k)}_j(\tilde \theta^r)}{2}\right\}\right).\]
The gaps are equal to $\Delta^{(k)}_j(\tilde \theta^r) = \max\{0, \frac{2}{3}(X_{\ihat_{(k)}} - r - X_j)\} = \max\{0, \frac{2}{3}( \hDelta_j^{(k)} - r)\}$. Since $r_{\max}$ is the maximum $r$ such that $r\leq r_\alpha(\tilde \theta^r)$, by the continuity of the noise we know that it must satisfy the boundary condition $r_{\max} = r_\alpha(\tilde \theta^{r_{\max}})$. This, finally, implies the claim of the lemma:
\[\alpha = \sum_{j=1}^m S\left(\max\left\{r_{\max}, \frac{\hDelta_j^{(k)} - r_{\max}}{3}\right\}\right).\]
\end{proof}

\section{Proofs from Section \ref{sec:population_winner}}

\subsection{Proof of Proposition \ref{prop:population_winner}}

\populationwinner*

\begin{proof}
$(t\in \widehat C^\alpha_{\ihat} \Rightarrow t\in \widehat C^\alpha_{*})$ Consider the ``worst-case'' vector $\theta^t$ from Eq.~\eqref{eq:worst-case-theta}. Lemma \ref{lemma:mainlemma} implicitly shows that, if $t\in \widehat C^\alpha_{\ihat}$, then $X\in A_\alpha(\theta^t)$. This is true because $|X_{\ihat} - \theta^t_{\ihat}| \leq r_\alpha(\theta^t)$ by the main claim of Lemma \ref{lemma:mainlemma}, and for all other indices $j$ we have $|X_j - \theta^t_j| \leq \max\{\frac{\Delta_j(\theta^t)}{2}, r_\alpha(\theta^t)\}$, as shown in Eq.~\eqref{eq:thetat_eq1} and Eq.~\eqref{eq:thetat_eq2}.
Since $t$ is the population winner at $\theta^t$, then that immediately gives $t\in \widehat C^\alpha_{*}$, as desired.\\

\noindent $(t\in \widehat C^\alpha_{*} \Rightarrow t\in \widehat C^\alpha_{\ihat})$
Take any $\theta$ such that $X\in A_\alpha(\theta)$. Let $\thetaswap$ be equal to $\theta$, except that we swap the values at $i^*$ and $\ihat$: $\thetaswap_{\ihat} = \theta_{i^*}$ and $\thetaswap_{i^*} = \theta_{\ihat}$. Note that throughout the proof $i^*$ denotes the population winner at $\theta$, not $\thetaswap$.
We argue that $X\in A_\alpha(\thetaswap)$. Since $\Delta_j(\theta) = \Delta_j(\thetaswap)$ for $j\not\in\{\ihat, i^*\}$ and
$r_\alpha(\thetaswap)=r_\alpha(\theta)$ by the symmetry of $\mathbf{S}$, it suffices to focus on the winning indices and show
\begin{equation}
\label{eq:swap}|X_{i} - \thetaswap_{i}| \leq r_\alpha(\thetaswap) \text{ for } i\in\{\ihat, i^*\}.
\end{equation}
At coordinate $\ihat$ we have:
\begin{align*}
|X_{\ihat} - \thetaswap_{\ihat}| &= \max\{X_{\ihat} - \thetaswap_{\ihat}, \thetaswap_{\ihat} - X_{\ihat}\}\\
&= \max\{X_{\ihat} - \theta_{i^*}, \theta_{i^*} - X_{\ihat}\} 
\leq \max\{X_{\ihat} - \theta_{\ihat}, \theta_{i^*} - X_{i^*}\} \leq r_\alpha(\theta),
\end{align*}
since $X\in A_\alpha(\theta)$. Similarly, at coordinate $i^*$ we have:
\begin{align*}
|X_{i^*} - \thetaswap_{i^*}| &= \max\{X_{i^*} - \thetaswap_{i^*}, \thetaswap_{i^*} - X_{i^*}\}\\
&= \max\{X_{i^*} - \theta_{\ihat}, \theta_{\ihat} - X_{i^*}\} 
\leq \max\{X_{\ihat} - \theta_{\ihat}, \theta_{i^*} - X_{i^*}\} \leq r_\alpha(\theta).
\end{align*}
This shows Eq.~\eqref{eq:swap}, and thus that $X\in A_\alpha(\thetaswap)$.

 If $t\in \widehat C^\alpha_{*}$, then that means that there exists a $\theta$ such that $\theta_{i^*}=t$ and $X\in A_\alpha(\theta)$. Since we have shown $X\in A_\alpha(\thetaswap)$ and $\thetaswap_{\ihat} = t$, then that means that $t \in \widehat C^\alpha_{\ihat}$.
\end{proof}

\subsection{Proof of Proposition \ref{prop:inference-on-index}}

\inferenceindex*

\begin{proof}
Using the fact that $t_l = X_{\ihat} - r_\alpha(\theta^{t_l})$, we know 
\[\{i\in[m] : X_i \geq X_{\ihat} - 2 r_\alpha(\theta^{t_l})\} = \{i\in[m] : X_i + r_\alpha(\theta^{t_l}) \geq t_l\}.\]
In the rest of the proof, we show that $\widehat \I^\alpha$ is equivalent to the set on the right-hand side.\\

\noindent $(i \in \widehat \I^\alpha \Rightarrow X_i + r_\alpha(\theta^{t_l}) \geq t_l)$ If $i \in \widehat \I^\alpha$, then there exists a vector $\theta$ such that $X\in A_\alpha(\theta)$ and $\theta_i \geq \theta_{\ihat}$. By Lemma \ref{lemma:mainlemma}, we know that $\theta_{\ihat} \geq t_l$. By the fact that $\theta_i$ is active, since $\Delta_i(\theta)=0$, we know $X_i + r_\alpha(\theta) \geq \theta_i$. Putting these facts together gives:
$$X_i + r_\alpha(\theta) \geq \theta_i \geq \theta_{\ihat} \geq t_l.$$
Since $r_\alpha(\theta^{t_l})$ is the maximum active radius across all $\theta$ with $X\in A_\alpha(\theta)$, we have $r_\alpha(\theta^{t_l}) \geq r_\alpha(\theta)$ and thus we have proved the desired claim.\\

\noindent $(i \in \widehat \I^\alpha \Leftarrow X_i + r_\alpha(\theta^{t_l}) \geq t_l)$ First, notice that $X_i + r_\alpha(\theta^{t_l}) \geq t_l$ implies that $i$ must be an active index at $\theta^{t_l}$. This is true because this inequality immediately implies $r_\alpha(\theta^{t_l}) \geq \frac 1 3 (t_l - X_i) = \Delta_i(\theta^{t_l})/2$. Next, Lemma \ref{lemma:r_only_inactive} (i) implies that we can define $\theta'$ that is identical to $\theta^{t_l}$, except that $\theta'_i = t_l$, without changing the active radius: $r_\alpha(\theta^{t_l}) = r_\alpha(\theta')$. Since 
\[- r_\alpha(\theta') \leq  X_i - t_l \leq X_{\ihat} - t_l \leq r_\alpha(\theta'),\]
we have $X\in A_\alpha(\theta')$. However, since $i$ is one of the population winners at $\theta'$, we have shown $i\in \widehat \I^\alpha$.
\end{proof}

\section{Proofs from Section \ref{sec:near-winners}}

\subsection{Proof of Proposition \ref{prop:near-winners}}

\nearwinners*

\begin{proof}
Recall from Proposition \ref{prop:population_winner} that $\widehat C_{i^*}^\alpha = \widehat C_{\ihat}^\alpha$ is a confidence set for the population winner $\theta^*$.
We stratify the set $\widehat C_j^\alpha$ based on the value of the population winner:
\begin{align*}
\widehat C_j^\alpha &= \left\{t: \theta \in \widehat C^\alpha, \theta_j = t \right\} = \cup_{s\in \widehat C_{i^*}^\alpha} \left\{t: \theta \in \widehat C^\alpha, \theta_j = t , \theta^* = s\right\}.
\end{align*}
We next further stratify based on whether coordinate $j$ is active or not at $\theta$:
\begin{align*}
\widehat C_j^\alpha &= \left(\cup_{s\in \widehat C_{i^*}^\alpha} \left\{t: \theta \in \widehat C^\alpha, \theta_j = t , \theta^* = s, j\in \I_\alpha(\theta)\right\} \right) \cup \left(\cup_{s\in \widehat C_{i^*}^\alpha} \left\{t: \theta \in \widehat C^\alpha, \theta_j = t , \theta^* = s, j\not \in \I_\alpha(\theta)\right\} \right)\\
&:= \widehat C_j^{\I,\alpha} \cup \widehat C_j^{\I^c,\alpha}.
\end{align*}
For the active part, we have
\begin{align*}
\widehat C_j^{\I,\alpha} &= \cup_{s\in \widehat C_{i^*}^\alpha} \left\{t: \theta \in \widehat C^\alpha, \theta_j = t , \theta^* = s, s \geq t \geq s - 2r_\alpha(\theta), |t - X_j|\leq r_\alpha(\theta)\right\}\\
&= \cup_{s\in \widehat C_{i^*}^\alpha} \left\{t: \theta \in \widehat C^\alpha , \theta^* = s, s \geq t \geq s - 2r_\alpha(\theta), |t - X_j|\leq r_\alpha(\theta)\right\}\\
&= \cup_{s\in \widehat C_{i^*}^\alpha} \left\{t: s \geq t \geq s - 2r_\alpha(\theta^s), |t - X_j|\leq r_\alpha(\theta^s)\right\}\\
&= \cup_{s\in \widehat C_{i^*}^\alpha} \left\{t: \min\{s, X_j + r_\alpha(\theta^s)\} \geq t \geq \max\{s - 2r_\alpha(\theta^s), X_j-r_\alpha(\theta^s)\}\right\},
\end{align*}
where the third step uses the ``worst-case'' vector \eqref{eq:worst-case-theta}, since among all vectors with $\theta^* = s$ this is the one that maximizes the active radius. Taking the worst case over $s\in \widehat C_{i^*}^\alpha \subseteq [X_{\ihat} - r_l, X_{\ihat} + r_u]$, we get
\begin{equation}
\label{eq:active_part}
\widehat C_j^{\I,\alpha} \subseteq \left[X_{\ihat} - 3r_l, X_{\ihat} + r_u\right]\cap \left[X_j - r_l, X_j + r_l\right].
\end{equation}
For the inactive part, we have
\begin{align*}
\widehat C_j^{\I^c,\alpha} &= \cup_{s\in \widehat C_{i^*}^\alpha} \left\{t: \theta \in \widehat C^\alpha, \theta_j = t , \theta^* = s, t \leq s - 2r_\alpha(\theta), |X_j - t|\leq \frac{s-t}{2}\right\}\\
&\subseteq \cup_{s\in \widehat C_{i^*}^\alpha} \left\{t: t \leq s, |X_j - t|\leq \frac{s-t}{2}\right\}\\
&= \cup_{s\in \widehat C_{i^*}^\alpha} \left\{t: t \leq s,~ 2X_j - s \leq t \leq X_j + \frac{s-X_j}{3}\right\}.
\end{align*}
Again taking the worst case over $s\in \widehat C_{i^*}^\alpha$, attained for $s = X_{\ihat} + r_u$, we have
\begin{equation}
\label{eq:inactive_part}
\widehat C_j^{\I^c,\alpha} \subseteq \left[X_j - \hDelta_j - r_u, X_j + \frac{1}{3}(\hDelta_j + r_u)\right].
\end{equation}
Putting together Eq.~\eqref{eq:active_part} and Eq.~\eqref{eq:inactive_part} completes the proof.
\end{proof}

\section{Proofs from Section \ref{sec:variance-adaptive}}

\subsection{Technical lemmas}
\label{sec:lemmas-sigmas}

Below we state the technical lemmas used in the proof of Lemma \ref{lemma:main-lemma-sigmas}.

\begin{lemma}[Large observations must be active (variance-adaptive)]
\label{lemma:winner_active-sigmas}
Suppose $X \in A_\alpha$, and let $\ihat = \argmax_{j\in[m]} \X_j$ and $\theta^* = \max_{j\in[m]} \theta_j$. Then,
\begin{itemize}
\item[(i)] $\ihat \in \I_\alpha$;
\item[(ii)] $j \in \I_\alpha$ for all $j\in\{k:X_k \geq \theta^*\}$.
\end{itemize}
 
\end{lemma}

\begin{proof}
First we prove claim (i).
Let $i^* = \argmax_{j\in[m]}\theta_j$, and note that $i^*$ must be active because $\Delta_{i^*}=0$. Now suppose that, in fact, $\ihat\in \I_\alpha^c$. Then, we know
\[\X_{\ihat} \leq \theta_{\ihat} + \max\left\{r_\alpha \sigma_{\ihat}, \Delta_{\ihat} \frac{\sigma_{\ihat}}{\sigma_{\ihat} + \sigma_{i^*}} \right\} = \theta_{\ihat} + \Delta_{\ihat} \frac{\sigma_{\ihat}}{\sigma_{\ihat} + \sigma_{i^*}}.\]
At the same time, because $i^*$ is active, we know
\[\X_{i^*} \geq \theta_{i^*} - r_\alpha \sigma_{i^*}.\]
Putting the two inequalities together, we get
\[\X_{\ihat} - \X_{i^*} \leq \theta_{\ihat} + \Delta_{\ihat} \frac{\sigma_{\ihat}}{\sigma_{\ihat} + \sigma_{i^*}} - \theta_{i^*} + r_\alpha \sigma_{i^*} = -\Delta_{\ihat}\frac{\sigma_{i^*}}{\sigma_{\ihat} + \sigma_{i^*}} + r_\alpha \sigma_{i^*}.\]
But $\X_{\ihat} - \X_{i^*}\geq 0$ by the definition of $\ihat$, implying $r_\alpha \geq \Delta_{\ihat}\frac{1}{\sigma_{\ihat} + \sigma_{i^*}}$ and thus contradicting the claim $\ihat\in\I_\alpha^c$.

Next, we look at claim (ii). Suppose that, in fact, some $j\in\{k:X_k \geq \theta^*\}$ is inactive, $j\in \I_\alpha^c$. Since $X\in A_\alpha$, we know
\[\theta_{j} + \Delta_{j} \frac{\sigma_{j}}{\sigma_{j} + \sigma_{i^*}} \geq X_j \geq \theta^* > \theta^* - r_\alpha \sigma_{i^*}.\]
Rearranging, this implies $r_\alpha \sigma_{i^*} > \Delta_j\frac{\sigma_{i^*}}{\sigma_j + \sigma_{i^*}}$. But indices $j$ with $r_\alpha > \Delta_j\frac{1}{\sigma_j + \sigma_{i^*}}$ are active, which contradicts the assumption that $j\in \I_\alpha^c$. 
\end{proof}

\begin{lemma}[Only inactive gaps matter (variance-adaptive)]
\label{lemma:r_only_inactive-sigmas}
Fix $\theta \in \R^m$ and let $\theta'\in \R^m$ be a vector with $\argmax_i \theta'_i = \argmax_i \theta_i = i^*$.
\begin{itemize}
\item[(i)] If $\theta'$ has $\Delta_j(\theta') = \Delta_j(\theta)$ for $j\in \I_\alpha^c(\theta)$ and $\Delta_j(\theta') \leq \Delta_j(\theta)$ for $j\in \I_\alpha(\theta)$, then $r_\alpha(\theta) = r_\alpha(\theta')$.
\item[(ii)] If $\theta'$ has $\Delta_j(\theta') \geq \Delta_j(\theta)$ for $j\in \I_\alpha^c(\theta)$, then $r_\alpha(\theta') \leq r_\alpha(\theta)$.
\end{itemize}
\end{lemma}

\begin{proof}
Recall that $r_\alpha(\theta) = \min\{r: \mathbf{S}(\max\{r\sigma,\Delta(\theta)\sigma/(\sigma + \sigma_{i^*})\}) \leq \alpha\}$.

\emph{Proof of (i).} Since $\Delta_j(\theta') \leq \Delta_j(\theta)$ for all $j$,  $\max\{r\sigma,\Delta_j(\theta')\sigma/(\sigma + \sigma_{i^*})\} \leq \max\{r\sigma ,\Delta_j(\theta)\sigma/(\sigma + \sigma_{i^*})\}$ for all $r$, and so $r_\alpha(\theta')$ has to be at least as large as $r_\alpha(\theta)$. However, since for $j\in \I_\alpha(\theta)$ we know $\max\{r_\alpha(\theta)\sigma_j,\Delta_j(\theta)\sigma_j/(\sigma_j + \sigma_{i^*})\} = r_\alpha(\theta)$, then we also have $\max\{r_\alpha(\theta)\sigma_j,\Delta_j(\theta')\sigma_j/(\sigma_j + \sigma_{i^*})\} = r_\alpha(\theta)$ and $\mathbf{S}(\max\{r_\alpha(\theta)\sigma,\Delta(\theta')\sigma/(\sigma + \sigma_{i^*})\}) \leq \alpha$. Consequently, $r_\alpha(\theta)$ satisfies the definition of the active radius for $\theta'$, and so $r_\alpha(\theta') = r_\alpha(\theta)$.

\emph{Proof of (ii).} Let $\theta''$ be any vector with $\argmax_i \theta'' = i^*$ such that:
\[\Delta_j(\theta'')=\begin{cases} \Delta_j(\theta) & \text{if } j\in \I^c_\alpha(\theta);\\
0 & \text{if }  j \in \I_\alpha(\theta).
\end{cases}\]
Then, by claim (i), $r_\alpha(\theta) = r_\alpha(\theta'')$. However, since $\theta''$ has gaps that are at least as small as those of $\theta'$ coordinate-wise, we know $r_\alpha(\theta'') \geq r_\alpha(\theta')$. Therefore, $r_\alpha(\theta) \geq r_\alpha(\theta')$.
\end{proof}

\begin{lemma}[Increasing inactive coordinates cannot increase all their radii (variance-adaptive)]
\label{lemma:increasing_coord-sigmas}
Fix a vector $\theta\in\R^m$ and an index set $\I\subseteq \I_\alpha^c(\theta)$. Let $\theta'$ be any vector such that $\argmax_{j\in[m]}\theta = \argmax_{j\in[m]}\theta' = i^*$, $\Delta_j(\theta) \leq \Delta_j(\theta')$ for all $j\in\I_\alpha^c(\theta) \setminus \I$, and  $\Delta_j(\theta') < \Delta_j(\theta)$ for $j\in\I$. Let 
$i_0 = \argmax_{j\in\I} \Delta_{j}(\theta) \frac{1}{\sigma_{j} + \sigma_{i^*}}.$
Then, $\Delta_{i_0}(\theta)\frac{\sigma_{i_0}}{\sigma_{i_0} + \sigma_{i^*}} \geq \max\{r_\alpha(\theta') \sigma_{i_0}, \Delta_{i_0}(\theta')\frac{\sigma_{i_0}}{\sigma_{i_0} + \sigma_{i^*}}\}$.
\end{lemma}

\begin{proof}
Suppose the claim is not true. Then, since $\Delta_{i_0}(\theta') < \Delta_{i_0}(\theta)$, that must mean that
\[r_\alpha(\theta') > \Delta_{i_0}(\theta) \frac{1}{\sigma_{i_0} + \sigma_{i^*}}
= \max_{j\in\I} \Delta_{j}(\theta) \frac{1}{\sigma_{j} + \sigma_{i^*}} > r_\alpha(\theta),\]
where the last inequality follows from the fact that $\I$ are inactive at $\theta$. This means that
\[\max\left\{\Delta_j(\theta') \frac{1}{\sigma_{j} + \sigma_{i^*}}, r_\alpha(\theta')\right\} > \max\left\{\Delta_j(\theta) \frac{1}{\sigma_{j} + \sigma_{i^*}}, r_\alpha(\theta)\right\} \]
for all $j\in\I \cup \I_\alpha(\theta)$. Moreover, for the remaining coordinates, $j\in\I^c_\alpha(\theta)\setminus \I$, we know that 
\[\max\left\{\Delta_j(\theta') \frac{1}{\sigma_{j} + \sigma_{i^*}}, r_\alpha(\theta')\right\} \geq \Delta_j(\theta') \frac{1}{\sigma_{j} + \sigma_{i^*}} \geq \Delta_j(\theta) \frac{1}{\sigma_{j} + \sigma_{i^*}} = \max\left\{\Delta_j(\theta) \frac{1}{\sigma_{j} + \sigma_{i^*}}, r_\alpha(\theta) \right\}.\]
In conclusion, the radius of the confidence region $A_\alpha(\theta')$ is at least as large as the radius of the confidence region $A_\alpha(\theta)$ in all coordinates $j$, and for some coordinates it is strictly larger. That means $P(X\in A_\alpha(\theta')) > P(X\in A_\alpha(\theta)) \geq 1-\alpha$, which in turn means that we could decrease the radius $r_\alpha(\theta')$ while maintaining the coverage at level $1-\alpha$. This contradicts the assumption that $r_\alpha(\theta')$ is minimal, as per Eq.~\eqref{eq:active-radius-sigmas}; therefore, the claim of the lemma must be correct.
\end{proof}

\subsection{Proof of Lemma \ref{lemma:main-lemma-sigmas}}

\mainlemmasigmas*

\begin{proof}
$(\Leftarrow)$ We know that for some $t^*\geq t$ and $i^*\in[m]$, $|\X_{\ihat}-t| \leq r_\alpha(\theta^{(t,t^*,i^*)}) \sigma_{\ihat}$ and $|\X_j-t^*| \leq r_\alpha(\theta^{(t,t^*,i^*)}) \sigma_{j}$ for $j\in\{i^*\}\cup \{k: X_k \geq t^*\}$. By the definition of the projection interval $\widehat C_{\ihat}^\alpha$, it suffices to show that $\X \in A_\alpha(\theta^{(t,t^*,i^*)})$, since $\theta^{(t,t^*,i^*)}_{\ihat} = t$. In other words, we show that we have $|\X_j-\theta^{(t,t^*,i^*)}_j| \leq \max\{r_\alpha(\theta^{(t,t^*,i^*)}) \sigma_j, \Delta_j(\theta^{(t,t^*,i^*)}) \frac{\sigma_j}{\sigma_j + \sigma_{i^*}}\}$ for all $j\in[m]$. This is true for coordinates $\ihat$ and $j\in \{i^*\} \cup \{k: X_k \geq t^*\}$ by assumption:
\[
|\X_{\ihat} - \theta^{(t,t^*,i^*)}_{\ihat}| = |\X_{\ihat} - t| \leq r_\alpha(\theta^{(t,t^*,i^*)}) \sigma_{\ihat}; \quad |\X_{j} - \theta^{(t,t^*,i^*)}_{j}| = |\X_{j} - t^*| \leq r_\alpha(\theta^{(t,t^*,i^*)}) \sigma_{j}.
\]
For all other coordinates $j$, by construction we have
\begin{equation*}
|\X_j - \theta^{(t,t^*,i^*)}_j| = \Delta_j(\theta^{(t,t^*,i^*)}) \frac{\sigma_j}{\sigma_j + \sigma_{i^*}}.
\end{equation*}
Therefore, we have $|\X_j-\theta^{(t,t^*,i^*)}_j| \leq \max\{r_\alpha(\theta^{(t,t^*,i^*)}), \Delta_j( \theta^{(t,t^*,i^*)}_j)\frac{\sigma_j}{\sigma_j + \sigma_{i^*}}\}$ for all $j\in[m]$, and thus $X\in A_\alpha(\theta^{(t,t^*,i^*)})$.\\

\noindent $(\Rightarrow)$ The fact that $t\in \widehat C_{\ihat}^\alpha$ means that  there exists some $\theta\in\R^m$ with $\theta_{\ihat}=t$ for which $\X \in A_\alpha(\theta)$. Let $i^*$ and $t^*$ denote the index and value of the population maximum at $\theta$.
We use Lemma \ref{lemma:winner_active-sigmas}, which says that $\ihat$ and $\{k: X_k \geq t^*\}$ must be active indices for any $\theta$ such that $\X\in A_\alpha(\theta)$. Therefore, we know
\begin{equation*}
|\X_{\ihat} - \theta_{\ihat}| = |\X_{\ihat} - t| \leq  r_\alpha(\theta) \sigma_{\ihat},
\end{equation*}
and for $j\in \{k: X_k \geq t^*\}$,
$$|X_j - t^*| = X_j - t^* \leq X_j - \theta_j \leq r_\alpha(\theta) \sigma_j.$$
Since the population winner has $\Delta_{i^*}(\theta)=0$ and is thus also active, we have
\begin{equation}
\label{eq:rtheta3}
|\X_{i^*} - \theta_{i^*}| = |\X_{i^*} - t^*| \leq  r_\alpha(\theta) \sigma_{i^*}.
\end{equation}
Therefore, if we show $r_\alpha(\theta^{(t,t^*,i^*)}) \geq r_\alpha(\theta)$ the proof is complete, so now we focus on showing this fact.

Lemma \ref{lemma:r_only_inactive-sigmas} (ii) gives a key characterization of the active radius: $r_\alpha(\theta)$ is non-increasing in the \emph{inactive} gaps. Therefore, if we prove that $\theta^{(t,t^*,i^*)}$ has the smallest inactive gaps of all $\theta$ with $\theta_{\ihat} = t$ and $\theta_{i^*} = t^*$ such that $X\in A_\alpha(\theta)$, then the result follows. Formally, if $\Delta_j(\theta) \geq \Delta_j(\theta^{(t,t^*,i^*)})$ holds for all $j\in \I^c_\alpha(\theta^{(t,t^*,i^*)})$, then $r_\alpha(\theta^{(t,t^*,i^*)}) \geq r_\alpha(\theta)$.

Suppose this is not true: for some $j\in \I^c_\alpha( \theta^{(t,t^*,i^*)})$, $\Delta_j(\theta) < \Delta_j( \theta^{(t,t^*,i^*)})$ (note, $X_j< t^*$ because all $\{j:X_j \geq t^*\}$ are in $\I_\alpha(\theta^{(t,t^*,i^*)})$). This assumption would imply that $\theta_j >  \theta^{(t,t^*,i^*)}_j$; therefore,
\begin{equation*}
\label{eq:inactive_gap-sigmas}
\theta_j - \X_j >  \theta^{(t,t^*,i^*)}_j - \X_j = \Delta_j(\theta^{(t,t^*,i^*)}) \frac{\sigma_j}{\sigma_j + \sigma_{i^*}}.
\end{equation*}
Lemma \ref{lemma:increasing_coord-sigmas} shows that there must be \emph{at least one} $j$ for which $\Delta_j(\theta) < \Delta_j(\theta^t)$, such that $\Delta_j( \theta^t)/2 \geq  \max\{r_\alpha(\theta), \Delta_j(\theta)/2\}$. As a result, there must exist a coordinate $j$ such that $\theta_j - \X_j > \max\{r_\alpha(\theta), \Delta_j(\theta)/2\}$, which contradicts the claim that $X\in A_\alpha(\theta)$. Therefore, we must in fact have $\Delta_j(\theta) \geq \Delta_j( \theta^t)$ for all $j\in \I^c_\alpha( \theta^t)$.

Finally, $r_\alpha( \theta^t) \geq r_\alpha(\theta)$ follows by invoking Lemma \ref{lemma:r_only_inactive-sigmas} (ii). Going back to Eq. \eqref{eq:rtheta3} and applying this inequality completes the proof.
\end{proof}

\end{document}